%% file: bare_jrnl_new_sample4.tex
\documentclass[lettersize,journal]{IEEEtran}
\usepackage{amsmath,amsfonts}
\usepackage{algorithmic}
\usepackage{algorithm}
\usepackage{array}
\usepackage{textcomp}
\usepackage{stfloats}
\usepackage{url}
\usepackage{verbatim}
\usepackage{graphicx}
\usepackage{cite}
\usepackage{amsmath}
\usepackage{amssymb}
\usepackage[T1]{fontenc}
\usepackage[utf8]{inputenc}
\usepackage{latexsym}
\usepackage{microtype}
\usepackage{amsfonts, amssymb, amsmath, tikz, booktabs, multirow, amsthm}
\usepackage{subfigure}
\usepackage{pgfplotstable}
\usepackage{pgfplots}
\pgfplotsset{compat=1.14}
\usetikzlibrary{calc} 
\usepackage{enumitem}

\newtheorem{lem}{Lemma}

\newtheorem{thm}{Theorem}
\newtheorem{cor}{Corollary}
\newtheorem{hypo}{Hypothesis}
\definecolor{RYB3}{RGB}{253, 180, 98}

\begin{document}

\title{On SkipGram Word Embedding Models with Negative Sampling: Unified Framework and Impact of Noise Distributions}
\author{Dezhi Liu, Richong Zhang,\IEEEmembership{~Member,~IEEE}, Ziqiao Wang

\thanks{This work was supported by the National Natural Science Foundation of China (No. U23B2056), and in part by the Fundamental Research Funds for the Central Universities, and in part by the State Key Laboratory of Complex \& Critical Software Environment. \it{(Corresponding author: Richong Zhang.)}}
\thanks{Dezhi Liu and Richong Zhang are with the School of Computer Science and Engineering, Beihang University, Xueyuan Road 37, Beijing, 100191, China (e-mail:dezhi.liu@buaa.edu.cn;zhangrc@act.buaa.edu.cn).}
\thanks{Ziqiao Wang is with the School of Computer Science and Technology, Tongji University, 4800 Cao’an Highway, Shanghai, 201804, China (e-mail: ziqiaowang@tongji.edu.cn).}
}

\markboth{Journal of \LaTeX\ Class Files,~Vol.~14, No.~8, August~2021}%
{Shell \MakeLowercase{\textit{et al.}}: A Sample Article Using IEEEtran.cls for IEEE Journals}


\maketitle

\begin{abstract}
SkipGram word embedding models with negative sampling, or SGN in short, is an elegant family of word embedding models. In this paper, we formulate a framework for word embedding, referred to as Word-Context Classification (WCC), that generalizes SGN to a wide family of models. The framework, which uses some ``noise examples'', is justified through theoretical analysis. The impact of noise distribution on the learning of the WCC embedding models is studied experimentally, suggesting that the best noise distribution is, in fact, the data distribution, in terms of both the embedding performance and the speed of convergence during training. Along our way, we discover several novel embedding models that outperform existing WCC models.
\end{abstract}

\begin{IEEEkeywords}
 Word Embedding, Negative Sampling, Noise Distribution, Adaptive Learning
\end{IEEEkeywords}

\section{Introduction}
\IEEEPARstart{L}{earning}  distributed word representations, commonly known as word embeddings \cite{W2V}, have been fundamental to natural language processing (NLP). The field has progressed from shallow architectures like GloVe \cite{pennington2014glove} and SkipGram \cite{MikolovSCCD13} to deep contextualized models, including ELMo \cite{peters2018deep}, BERT \cite{devlin2019bert}, and more recently, decoder-only LLMs as generalist embedding models \cite{tao2024llms}. Despite these advances, pre-trained word embeddings remain essential for initializing NLP models, underscoring the continued importance of embedding learning research.

Among these approaches, SkipGram represents one of the earliest and most influential word embedding methods. The model operates by parsing text into (center word, context word) pairs and learning to predict context words from center words. This predictor is parameterized by the word embeddings themselves, such that learning the predictor yields the desired embeddings. However, with large vocabularies, this approach becomes computationally expensive as it requires a $V$-class classifier. Negative sampling addresses this issue by reformulating the problem as binary classification: genuine word-context pairs from the corpus serve as positive examples, while artificially generated pairs from a noise distribution provide negative examples. The resulting binary classifier, still parameterized by word embeddings, provides an efficient alternative for obtaining high-quality embeddings. We refer to this approach as SGN (SkipGram with Negative Sampling) throughout this paper.


Despite its simplicity, SGN has been shown to perform quite well in~\cite{MikolovSCCD13}. Even though SGN now tends to be replaced by more sophisticated embedding models in NLP applications, to us, SGN is still a fundamental model: its elegance seems to imply a general principle for representation learning; its simplicity may allow a theoretical analysis. On the contrary, in those more advanced and sophisticated models, e.g., BERT~\cite{devlin2019bert,brown2020language}, LLMs~\cite{tao2024llms,leenv}, the design is primarily motivated by conceptual intuitions; theoretical analysis appears very difficult, at least with current analytical tools developed for deep learning. 

However, the theoretical analysis of SGN remains thin to date. The main analytic result derived for SGN to date is that of \cite{LevyG14}. The result states that, under a particular noise distribution, the SGN can be interpreted as an implicit factorization of a pointwise mutual information matrix.

Many questions in fact are unanswered for SGN. Specifically, in this paper, we ask the following questions. Beyond that particular distribution, if one chooses a different noise distribution, is SGN still theoretically justified? Is there a general principle underlying SGN that allows us to build new embedding models?
If so, how does the noise distribution impact the training of such models and their achievable performances? 
These questions are answered in this paper. 

To that end, we formalize a unified framework, referred to as ``word-context classification'' or WCC, for SGN-like word embedding models. The framework, including SGN as a special case,  allows the noise distribution to take arbitrary forms and hence results in a broad family of SGN-like models.  We also provide a theoretical analysis that justifies the WCC framework. Consequently, the matrix factorization result of \cite{LevyG14} can also be derived from this analysis as a special case. 

In addition, the impact of noise distribution on learning word embeddings in WCC is also studied in this paper. 
For this purpose, we classify the WCC models into conventional SGN models and conditional SGN models, according to the factorization form of the noise distribution. We argue theoretically that the conditional models are inherently advantageous, thereby hypothesizing that the best WCC model is in fact the conditional SGN model with noise distribution precisely equal to the data distribution, i.e., the word-context pair distribution in the corpus. It is unfortunate, however, that the conditional models are in fact discouraged due to their high training complexity. 
To tackle this, we propose a variant of the conditional SGN model, the adaptive conditional SGN (caSGN) model, where the noise distribution adapts itself gradually towards the data distribution. This adaptation was achieved by generatively modeling the noise distribution and learning the generator in a way similar to that of GAN \cite{GoodfellowPMXWOCB14}. We show that the caSGN models may be constructed with various structures of the generator. In particular, a previously proposed model, ACE \cite{BoseLC18}, may be regarded a special form of caSGN.  For the sake of comparison, the adaptive version of the unconditional SGN model is also presented.

Extensive experiments are then conducted to further this study. A wide spectrum of WCC models (SGN, versions of caSGN, and ACE) are implemented and compared to investigate the impact of noise distribution. These learned embeddings are evaluated using three downstream tasks over 12 datasets. The experiments we conducted indicate that the caSGN models are superior to other models, thereby affirming the accuracy of the WCC framework and verifying the hypothesis that the most suitable noise distribution in WCC is the data distribution.

\section{Related Work}\label{relate work}

Negative sampling has been a key technique in word embedding learning since its introduction in Word2Vec~\cite{MikolovSCCD13}. Recent studies have improved negative sampling from various angles. For example, \cite{zhang2018gneg} proposed a graph-based method that uses global corpus information to generate word-specific noise distributions, improving word analogy and similarity performance. \cite{chen2018improving} addressed gradient vanishing in skip-gram by dynamically selecting informative negative samples based on inner product scores.  \cite{mu2019revisiting} rectified the skip-gram model with quadratic regularization. These works focus on objective function design and are less related to our focus on noise distribution.

Other techniques, such as GAN, have also been utilized in the learning of embeddings, and GAN-based negative sampling has gained attention in various domains \cite{Wang0017,wang2017irgan,wang2018incorporating, GoodfellowPMXWOCB14}. \cite{Wang0017} used GANs to generate fake sentences for commonsense machine comprehension. \cite{wang2017irgan} proposed a minimax game combining generative and discriminative retrieval models.\cite{wang2018incorporating} used GANs to generate negative samples for knowledge representation learning.~\cite{GoodfellowPMXWOCB14} applied GANs to text generation via word embeddings.

Our work distinguishes itself by providing a unified theoretical framework that systematically analyzes the role of noise distributions in embedding learning. While recent studies have explored adaptive sampling strategies in specific domains \cite{bai2022open} and generalized Skip-Gram formulations \cite{zhu2023free}, our WCC framework offers a comprehensive theoretical foundation that encompasses these approaches, with particular emphasis on establishing the optimality conditions for noise distributions in word embedding tasks.

Negative sampling plays a vital role in contrastive learning for sentence and graph embeddings. \cite{chen_incremental_2022} proposed incrementally removing false negatives, while \cite{wang_sncse_2022} introduced soft negative samples to improve sentence embeddings. \cite{robinson_contrastive_2021} emphasized the importance of hard negatives, and \cite{kalantidis_hard_2020} suggested hard negative mixing at the feature level. \cite{zhang_unsupervised_2022} focused on creating diverse positives and negatives at the group level. In graph embedding, \cite{liu2025bypassing} recently proposed dimension regularization as a more efficient alternative to skip-gram negative sampling. Notably, the well-known noise contrastive estimation (NCE) framework shares a similar negative sampling approach, though its connection to methods like SGN and WCC is loose, as NCE requires an evaluable noise distribution, unlike WCC.

Theoretical analyses of negative sampling have been limited. \cite{LevyG14} showed that SGN implicitly factorizes a PMI matrix. \cite{zhang2021understanding} examined hard negatives in noise contrastive estimation. \cite{oyama2023norm} illustrated that the norm of word embeddings encodes information gain related to noise distribution. \cite{han2024word} indicated that word embeddings can influence language models, emphasizing embedding quality. \cite{yang2024does} provided a review of negative sampling's theory and applications in machine learning. Our work builds on these advances by proposing a unified framework that generalizes SGN and allows for adaptive noise distributions, supported by theoretical and empirical evidence.

\section{Word-Context Classification}\label{method}

\subsection{The Word Embedding Problem}
Let ${\cal X}$ denote a vocabulary of words, and let ${\cal Y}$ denote a set of contexts. When considering the SkipGram models, context is often considered as a single word, which is the case ${\cal Y}$ is ${\cal X}$. But to better distinguish words and contexts, we prefer to use ${\cal Y}$ to denote the set of all contexts. 

In this setting, a given training corpus may be parsed into a collection ${\cal D}^+$ of word-context pairs $(x, y)$ from ${\cal X}\times {\cal Y}$. For example, as is standard \cite{MikolovSCCD13}, we may use a running window of length $2L+1$ to slide over the corpus; for each window location, we collect $2L$ word-context pairs, where the word located in the center of the window is taken as $x$ in each pair $(x, y)$, and each of the remaining $2L$ words in the window is taken as a context $y$, paired with $x$. This gives rise to the {\em training data}, or the {\em positive sample} ${\cal D}^+$. With respect to context $y$, we sometimes call the word $x$ the ``center word''.

As is conventional, word-context pairs ${\cal D}^+$ are assumed to
contain i.i.d. instances of a random word-context pair $(X, Y)$ drawn from an unknown distribution ${\mathbb P}$ on ${\cal X}\times {\cal Y}$.
We will use ${\mathbb P}_{\cal X}$ to denote the marginal of ${\mathbb P}$ on ${\cal X}$, and for each $x\in {\cal X}$, use ${\mathbb P}_{{\cal Y}|{x}}$ 
to denote the conditional distribution of $Y$ given $X=x$ under ${\mathbb P}$.
Let $N^+$ denote the number of pairs in ${\cal D}^+$. 

The objective of word embedding is to learn a vector representation for each word in ${\cal X}$ (and possibly also for each context in ${\cal Y}$).

We now introduce the Word-Context Classification (WCC) framework, which provides a unified perspective to negative-sampling based SkipGram embedding models.

\subsection{The WCC Framework}\label{sec:32}

For each $x\in {\cal X}$, we let ${\mathbb Q}_{{\cal Y}|x}$ be a distribution on ${\cal Y}$, and we define a distribution ${\mathbb Q}$ on ${\cal X} \times {\cal Y}$ as the {\em noise distribution}.
Given ${\mathbb Q}$, we draw $N^-$ word-context pairs i.i.d. from ${\mathbb Q}$ to form a {\em noise sample} or {\em negative sample} ${\cal D}^-$. 

Now we define a binary classification problem on samples ${\cal D}^+$ and ${\cal D}^-$ with the objective of learning a binary classifier capable of distinguishing the word-context pairs drawn from ${\mathbb P}$ from those drawn from ${\mathbb Q}$.  The word embedding problem of interest will be nested inside the classifier-learning problem.

To that end, let the binary variable $U$ denote the class label associated with each pair of words-context $(X, Y)$. Specifically, whenever we draw a pair $(X, Y)$ from ${\mathbb P}$, we also create a label $U\!=\!1$, and likewise, whenever we draw a pair $(X, Y)$ from ${\mathbb Q}$, we also create a label $U\!=\!0$. That is, all pairs in ${\cal D}^+$ are associated with label $U\!=\!1$, and all pairs in ${\cal D}^-$ associated with label $U\!=\!0$.  Then the classification problem is equivalent to learning the conditional distribution $p_{U|XY}(\cdot|x, y)$ from ${\cal D}^+$ and ${\cal D}^-$. 

Let $\sigma(\cdot)$ denote the logistic function and let the classifier  $p_{U|XY}(\cdot|x,y)$ take the form
\begin{equation}
p_{U|XY}(1|x,y) := \sigma \left(s(x, y)\right)
\end{equation}
for some function $s$ on ${\cal X}\times {\cal Y}$. Note that such a form of classifiers is well known to be sufficiently expressive and entail no loss of generality
\cite{jordan1995logistic}. We will refer to $s(x, y)$ as the {\em score} of the word-context pair $(x, y)$. Any parameterized modelling for such a classification problem then reduces to a selection of the score function $s(\cdot)$. 

In order to learn a distributed representation of words, consider the following family of parameterizations of the score function $s$. 

Let $\overline{\cal X}$ and $\overline{\cal Y}$ be two vector spaces that serve as embedding spaces for words and contexts, respectively. Let $f:{\cal X} \rightarrow \overline{\cal X}$  and $g:{\cal Y} \rightarrow \overline{\cal Y}$ be two functions representing the embedding maps for words and contexts. Let $s(x, y)$ take the form
\begin{equation}
s(x, y):= {\bf score}\left(f(x), g(y)\right),
\end{equation}
for some function ${\bf score }(\cdot)$ on $\overline{\cal X}\times \overline{\cal Y}$. In the most general case, the functions $f$, $g$ and ${\bf score}$ can all be made learnable.  In this paper, however, we
follow the classical choice in \cite{MikolovSCCD13} for simplicity, where
$\overline{\cal X}$ and $\overline{\cal Y}$ are taken as the same vector space, and
the function ${\bf score}(\cdot)$ is taken as the standard inner product operation therein, namely not to be learned.

It is easy to see that the standard cross-entropy loss for this classification problem is 
\begin{equation}
\label{eq:loss}
\ell \!\!=\!-\!\!\!\!\!\sum_{(x, y)\in {\cal D}^+}\!\!\!\!\!\ \log  \sigma \left(s(x, y)\right) -\!\!\!\!\!\sum_{(x, y)\in {\cal D}^-}\!\!\!\!\!\ \log  \sigma \left( - s(x, y)\right).
\end{equation}
The standard Maximum Likelihood formulation of the learning problem is then minimizing the loss function $\ell$ over all possible embedding maps $f$ and $g$, namely, solving for
\begin{equation}
\label{eq:opt}
    \left({f^*}, {g^*}\right): = 
    \arg\min_{f, g} \ell (f, g).
\end{equation}
Above, we have explicitly written the cross-entropy loss $\ell$ as a function of the embedding maps $f$ and $g$.

At this end, we see that solving this binary ``word-context classification'' problem provides an embedding map $f$, thus giving a solution to the word embedding problem of interest. 
We refer to this framework as the Word-Context Classification or WCC framework. 

\subsection{Theoretical Properties of WCC}

To justify the WCC framework, we derive a set of theoretical properties for the optimization problem in (\ref{eq:opt}).

Let $\widetilde{\mathbb P}$ and $\widetilde{\mathbb Q}$ be the empirical word-context distributions observed in ${\cal D}^+$ and ${\cal D}^-$ respectively. That is, $\widetilde{\mathbb P}(x,y) = \frac{\#(x, y)}{N^+}$ where $\#(x, y)$ is the number of times the word-context pair $(x, y)$ appears in ${\cal D}^+$, and  $\widetilde{\mathbb Q}(x,y)$ is defined similarly.

We will say that the distribution $\widetilde{\mathbb Q}$ {\em covers} the distribution $\widetilde{\mathbb P}$ if the support ${\rm Supp}\left( \widetilde{\mathbb P} \right)$ of $\widetilde{\mathbb P}$ is a subset of the support
${\rm Supp}\left( \widetilde{\mathbb Q} \right)$
of $\widetilde{\mathbb Q}$.  Recall that the support of a distribution is a set of all points on which the probability is non-zero.

Note that the function $s$ assigns a score to each word-context pair $(x, y)$. Thus, we can view $s$ as an $|{\cal X}| \times |{\cal Y}|$ ``score matrix''. Additionally, the loss $\ell$ in (\ref{eq:loss}) may also be treated as a function of the matrix $s$. 

\begin{thm}
\label{thm:matrixFactorization}
Suppose that $\widetilde{\mathbb Q}$ covers $\widetilde{\mathbb P}$. Then the following holds.
\begin{enumerate}
\item The loss $\ell$, as a function of $s$, is convex in $s$. 
\item If $f$ and $g$ are sufficiently expressive, then there is a unique configuration $s^*$ of $s$ that minimizes $\ell(s)$, and the global minimizer $s^*$ of $\ell(s)$ is given by 
\[
s^*(x, y) = \log
\frac{\widetilde{\mathbb P}(x, y)}{\widetilde{\mathbb Q}(x, y)} + \log \frac{N^+}{N^-}
\]
for every $(x, y) \in {\cal X}\times {\cal Y}$.
\end{enumerate}
\end{thm}

\begin{proof}
 The first part is straightforward by computing the second derivative:
	\begin{equation*}
	\frac{\partial^2 \left(-\log  \sigma (s)\right)}{\partial s^2}=\frac{\partial  \left(\sigma(s)-1\right)}{\partial s}= \sigma(s) \cdot \left(1-\sigma(s)\right) \geqslant 0.
	\end{equation*}
	Thus, $-\log  \sigma \left(s\right)$ is a convex function in $s$, and we can prove $-\log  \sigma \left(-s\right)$ is also a convex function following the same way. Since the summation of the convex functions is still a convex function, the loss $\ell$ is convex in $s$. 

 For the second part, we know that the size of ${\cal D}^+$ is ${N^+}\widetilde{\mathbb P}(x,y)$ and the size of ${\cal D}^-$ is $ {N^-}\widetilde{\mathbb Q}(x,y)$. Therefore, the prior distribution of the binary random variable $U$ is:
	\begin{eqnarray}
	\label{eq:pu}
	p_U(U=1)=\frac{{N^+}\widetilde{\mathbb P}(x,y)}{{N^+}\widetilde{\mathbb P}(x,y)+{N^-}\widetilde{\mathbb Q}(x,y)}.
	\end{eqnarray}
	Recall that 
	\[
	\ell - {\rm H}(p_U)={\rm KL}(p_U || p_{U|XY}),
	\]
	 where ${\rm H}(p_U)$ is the entropy of $p_U$ and ${\rm KL}(p_U || p_{U|XY})$ is the Kullback-Leibler divergence between $p_U$ and $p_{U|XY}$. Given ${N^+}$, $ {N^-}$, $\widetilde{\mathbb P}(x,y)$ and $\widetilde{\mathbb Q}(x,y)$, ${\rm H}(p_U)$ is a constant. In this case, 
	 \[
    \min_{f, g} \ell (f, g) = \min_{f, g} {\rm KL}(p_U || p_{U|XY}).
    \]
	 
	 Since ${\rm KL}(p_U || p_{U|XY})$ reaches the minimum value $0$ when $p_U = p_{U|XY}$, we let
		\begin{eqnarray}
		p_U(U=1)&=&p_{U|XY}(1|x,y)\notag\\\notag
		&\Longrightarrow&
		\frac{{N^+}\widetilde{\mathbb P}(x,y)}{{N^+}\widetilde{\mathbb P}(x,y)+{N^-}\widetilde{\mathbb Q}(x,y)}\\
		&= &\sigma \left(s^*(x, y)\right), 
		\end{eqnarray}
	which indicates $s^*(x, y) = \log
	\frac{\widetilde{\mathbb P}(x, y)}{\widetilde{\mathbb Q}(x, y)} + \log \frac{N^+}{N^-}$. We then know that this is the unique solution due to the convexity.
\end{proof}

Note that the second statement of the theorem does not imply that there is a unique $\left({f^*}, {g^*}\right)$ that minimizes $\ell (f, g)$. In fact, there is a continuum of $\left({f^*}, {g^*}\right)$'s which all minimize $\ell(f, g)$ equally well. 
A consequence of Theorem \ref{thm:matrixFactorization} is the following result.

\begin{cor}
\label{cor:justifyWCC}
Let $N^+=n$ and $N^-=kn$. Suppose that ${\mathbb Q}$ covers ${\mathbb P}$, and that $f$ and $g$ are sufficiently expressive.
Then it is possible to construct a distribution $\widehat{\mathbb P}$ on ${\cal X}\times {\cal Y}$ using ${f^*}, {g^*}$, $k$, and ${\mathbb Q}$ such that for every $(x, y)\in {\cal X}\times {\cal Y}$, $\widehat{\mathbb P}(x, y)$ converges to ${\mathbb P}(x, y)$ in probability as $n\rightarrow \infty$.
\end{cor}


\begin{proof}

Suppose ${f^*}, {g^*}$, $k$, and ${\mathbb Q}$ are all known. Recall that $s^*(x, y) = \log
\frac{\widetilde{\mathbb P}(x, y)}{\widetilde{\mathbb Q}(x, y)} - \log k$ and $s^*(x,y)=\langle f^*(x), g^*(y)\rangle$. We then reconstruct $\widehat{\mathbb P}$ as
\begin{equation}
\begin{split}
\widehat{\mathbb P}(x,y) &= \underbrace{\exp \left\{ \langle f^*(x), g^*(y)\rangle + \log k \right\}}_{A(x,y)} \cdot \mathbb{Q}(x,y)\\
&=A(x,y) \cdot \widetilde{\mathbb Q}(x,y) \cdot \frac{\mathbb{Q}(x,y)}{\widetilde{\mathbb Q}(x,y)}\\
&=\widetilde{\mathbb P}(x,y)\cdot \frac{\mathbb{Q}(x,y)}{\widetilde{\mathbb Q}(x,y)}\\
&=\mathbb{P}(x,y) \cdot \frac{\widetilde{\mathbb P}(x,y)}{\mathbb{P}(x,y)} \cdot \frac{\mathbb{Q}(x,y)}{\widetilde{\mathbb Q}(x,y)}.
\end{split}
\end{equation}
According to the weak law of large numbers, $\widetilde{\mathbb P}(x, y)$ converges to ${\mathbb P}(x, y)$ in probability as $n\rightarrow \infty$ and $\widetilde{\mathbb Q}(x, y)$ converges to ${\mathbb Q}(x, y)$ in probability as $n\rightarrow \infty$. Thus, for every $(x, y)\in {\cal X}\times {\cal Y}$, $\widehat{\mathbb P}(x, y)$ converges to ${\mathbb P}(x, y)$ in probability as $n\rightarrow \infty$
\end{proof}

We note that Corollary \ref{cor:justifyWCC} justifies the WCC framework, since under the condition of the corollary, one can reconstruct the unknown data distribution ${\mathbb P}$
from the learned embedding maps ${f^*}$ and ${g^*}$ without referring to the samples ${\cal D}^+$ and ${\cal D}^-$. That is, for sufficiently large sample sizes $N^+$ and $N^-$, the learned embedding maps 
${f^*}$ and ${g^*}$ results in virtually no loss of the information contained in the data distribution ${\mathbb P}$.

There are other consequences of Theorem \ref{thm:matrixFactorization}, which we postpone to discuss in a later section. For later use, we present another result.

\begin{lem} 
\label{lem:Gradient}
The derivative of the loss function $\ell$ with respect to $s(x, y)$ is
\[
\frac{\partial \ell}{\partial s(x, y)} \!\!=\!\!
\sigma \left(s(x, y)\right) (N^- \widetilde{\mathbb Q}(x, y)
-e^{-s(x, y)} N^+ \widetilde{\mathbb P}(x, y)   
).
\]
\end{lem}

\subsection{Different Forms of Noise Distribution ${\mathbb Q}$} \label{sec:3-4}

From the formulation of the WCC, it is apparent that the choice of the noise distribution ${\mathbb Q}$ has an effect on this classifier learning problem and possibly affects the quality and training of the word embeddings.  We now discuss various options in selecting the noise distribution ${\mathbb Q}$, resulting in different versions of the SkipGram models. These models will all be referred to SGN models, for the ease of reference.  

\subsubsection{SGN Model}
Let ${\mathbb Q}$ factorize in the following form
\begin{equation}
\label{eq:uncondSGN}
 {\mathbb Q}(x, y) = \widetilde{\mathbb P}_{\cal X}(x){\mathbb Q}_{\cal Y}(y).
\end{equation}

In such a model, the noise context does not depend on the noise center word. Thus, we may also refer to such a model as an unconditional SGN model. 
The standard SGN model of \cite{MikolovSCCD13}, which we will refer to as ``Vanilla SGN'',  can be regarded as an unconditional SGN in which ${\mathbb Q}_{\cal Y}$ takes a particular form (see later).

The following result follows from Theorem \ref{thm:matrixFactorization}.
\begin{cor}
\label{cor:pmi}
In an unconditional SGN model, suppose that $f$ and $g$ are sufficiently expressive. Let $N^+\!\!\!=n$ and $N^-\!\!\!=kn$. Then
the global minimizer of loss function (\ref{eq:loss}) is given by
\[
\label{eq:uncondOptScore}
    s^*(x,y)=\overline{x} \cdot \overline{y}=\log
    \frac{\widetilde{\mathbb P}(x, y)}{\widetilde{{\mathbb P}}_{\cal X}(x) \widetilde{{\mathbb Q}}_{\cal Y}(y)} - \log k.
\]
\end{cor}
As a special case of Corollary \ref{eq:uncondOptScore}, when $\widetilde{{\mathbb Q}}_{\cal Y}=\widetilde{\mathbb P}_{\cal Y}$,   the term $\log
    \frac{\widetilde{\mathbb P}(x, y)}{\widetilde{{\mathbb P}}_{\cal X}(x) \widetilde{{\mathbb Q}}_{\cal Y}(y)}$ is called ``pointwise mutual information'' (PMI) \cite{LevyG14}. In this case, the above corollary states that the WCC learning problem can be regarded as implicitly factorizing a ``shifted PMI matrix'', recovering the well-known result of \cite{LevyG14} on SkipGram word embedding models. 

It is natural to consider the following forms of ${\mathbb Q}_{\cal Y}$ in unconditional SGN. 

\begin{enumerate}
    \item {\bf ``uniform SGN'' (ufSGN)}: Let ${\mathbb{Q} }_{\cal Y}$ be the discrete uniform distribution over ${\cal Y}$, i.e. ${\mathbb Q}_{\cal Y}(y)=1/{|{\cal Y}|}$.
    \item {\bf ``unigram SGN'' (ugSGN)}: Let ${\mathbb Q}_{\cal Y}$ be the empirical distribution ${\mathbb P}_{\cal Y}$ of the context word in the corpus, that is, ${\mathbb Q}_{\cal Y}(y) = f_y/\sum_{y\in {\cal Y}}{f_y}$, where $f_y$ is the frequency at which the context word $y$ has occurred in the corpus.
    \item {\bf ``3/4-unigram SGN'' (3/4-ugSGN)}: Let ${\mathbb Q}_{\cal Y}$ be defined by ${\mathbb Q}_{\cal Y}(y) = f_y^{3/4}/\sum_{y\in {\cal Y}}{f_y^{3/4}}$. This is precisely the noise distribution used in vanilla SGN \cite{MikolovSCCD13}.
\end{enumerate}

\subsubsection{Conditional SGN}
In this case, we factorize ${\mathbb Q}$ as
\[
 {\mathbb Q}(x, y) = \widetilde{\mathbb P}_{\cal X}(x) {\mathbb Q}_{{\cal Y}|x}(y),
\]
where ${\mathbb Q}_{{\cal Y}|x}(\cdot)$ varies with $x$. This form of ${\mathbb Q}$ includes all possible distributions ${\mathbb Q}$ whose marginals ${\mathbb Q}_{\cal X}$ on the center word are the same as $\widetilde{\mathbb P}_{\cal X}$. Specifically, if we take ${\mathbb Q}_{{\cal Y}|x}$ as $\widetilde{\mathbb P}_{{\cal Y}|x}$, then ${\mathbb Q}$ is $\widetilde{\mathbb P}$. Before proceeding, we make the following remark.

\noindent{\bf Remark 1.} \emph{In Theorem \ref{thm:matrixFactorization} and Corollary \ref{cor:justifyWCC}, the WCC framework is justified for any choice of empirical noise distribution $\widetilde{\mathbb Q}$ that covers $\widetilde{\mathbb P}$. Consider some $(x, y) \in {\rm Supp}\left( \widetilde{\mathbb Q} \right) \setminus {\rm Supp}\left( \widetilde{\mathbb P} \right)$, namely, $(x, y)$ is ``covered'' by $\widetilde{\mathbb Q}$ but not by $\widetilde{\mathbb P}$. By Lemma \ref{lem:Gradient}, the gradient is
\begin{eqnarray*}
\frac{\partial \ell}{\partial s(x, y)} 
 & = & 
 \sigma \left(s(x, y)\right)\cdot 
 N^- \widetilde{\mathbb Q}(x, y).
\end{eqnarray*}
That is, the gradient signal contains only one term, which is used to update the representation $(f(x), g(y))$ for the pair $(x, y)$, which is outside the positive examples. Although the gradient signal updates the embedding $f(x)$, the training aims to make the classifier sensitive to negative examples (making it have a low loss), without contributing to differentiating the negative examples from the positive ones (namely, without aiming at reducing the loss for positive examples).  Such a direction of training, not necessarily ``wrong'', somewhat deviates from the training objective (i.e., reducing the loss for both positive and negative examples). In the case when the ${\rm Supp}\left( \widetilde{\mathbb Q} \right)$ of 
$\widetilde{\mathbb Q}$ contains the support ${\rm Supp}\left( \widetilde{\mathbb  P} \right)$ of $\widetilde{\mathbb P}$ but is much larger, there is a significant fraction of such negative examples outside ${\rm Supp}\left( \widetilde{\mathbb P}\right)$. This may result in slow training.} 

Thus in the sense of efficiently learning word embeddings in the WCC framework, we have the following hypothesis:

\begin{hypo} 
\label{hp:QequalP}
 The best $\widetilde{\mathbb Q}$ is the one that barely covers $\widetilde{\mathbb P}$, namely, equal to $\widetilde{\mathbb P}$.
\end{hypo}

\noindent{\bf Remark 2.} \emph{It is important to note that in order to achieve better performance on some NLP downstream tasks, $\widetilde{\mathbb Q}$ should not be exactly $\widetilde{\mathbb P}$ during the whole training phase. This is because target words in these tasks may not frequently appear in the training corpus, and if they are rarely trained, the learned embeddings will not be able to give a good performance. It turns out that we usually apply the sub-sampling technique and try to improve the entropy of $\widetilde{\mathbb Q}$ (e.g., using ``3/4-unigram" instead of ``unigram" \cite{MikolovSCCD13} ) in practice.}

Under this hypothesis, we wish to choose ${\mathbb Q}_{{\cal Y}|x}$ to be equal to, or at least closely resemble, $\widetilde{\mathbb P}_{{\cal Y}|x}$.

This choice, however, entails nearly prohibitively high computational complexity for optimization using mini-batched SGD. This is because for each word-context pair in a mini-batch, the context word needs to be sampled from its own distribution $\widetilde{\mathbb P}_{{\cal Y}|x}$, depending on the center word $x$. Such a sampling scheme is not parallelizable within the mini-batch, under current deep learning libraries. 

In the next subsection, we will present a solution that can bypass this complexity. 

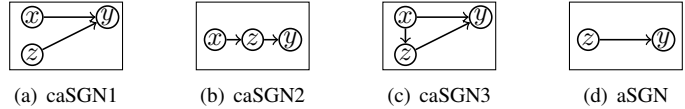
\begin{figure}[htbp]
\centering
    \tikzstyle{style_x}=[draw,semithick, shape=circle,minimum size=12,inner sep=0.1pt, black]
    \tikzstyle{style_z}=[draw,semithick, shape=circle,minimum size=12,inner sep=0.1pt, black]
    \tikzstyle{style_y}=[draw,semithick, shape=circle,minimum size=12,inner sep=0.1pt, black]
    \tikzstyle{line}=[semithick, <-]
   
    \subfigure [caSGN1] {\label{fig:c1} 
    \begin{tikzpicture}[scale=.95]
        \node[style_x] (x1) {$x$};
        \node[style_z] (z1) [below of=x1, node distance=1cm] { $z$};
        \node[style_y] (y1) [right of=x1, node distance=1.5cm] {$y$}
        edge [line] (x1)
        edge [line] (z1);
        \draw[draw=black] (-0.6,-1.4) rectangle ++(3,2);
    \end{tikzpicture}
    }
    \qquad
    \subfigure [caSGN2] {\label{fig:c2} 
    \begin{tikzpicture}[scale=.95]
        \node[style_x] (x2) {$x$};
        \node[style_z] (z2) [right of=x2, node distance=1cm] { $z$}
        edge [line] (x2);
        \node[style_y] (y2) [right of=z2, node distance=1cm] {$y$}
        edge [line] (z2);
        \draw[draw=black] (-0.5,-0.9) rectangle ++(3,2);
    \end{tikzpicture}
    }
    \qquad
    \subfigure [caSGN3]{ \label{fig:c3} 
    \begin{tikzpicture}[scale=.95]
        \node[style_x] (x3) {$x$};
        \node[style_z] (z3) [below of=x3, node distance=1cm] { $z$}
        edge [line] (x3);
        \node[style_y] (y3) [right of=x3, node distance=1.5cm] {$y$}
        edge [line] (x3)
        edge [line] (z3);
        \draw[draw=black] (-0.6,-1.4) rectangle ++(3,2);
    \end{tikzpicture}
    }
    \qquad
    \subfigure [aSGN] {\label{fig:u} 
    \begin{tikzpicture}[scale=.95]
        \node[style_x] (z0) {$z$};
        \node[style_y] (y0) [right of=z0, node distance=1cm] {$y$}
        edge [line] (z0);
        \draw[draw=black] (-0.5,-0.9) rectangle ++(3,2);
    \end{tikzpicture}
    }
\caption{Generators of the adaptive SkipGram model}
\label{fig:generator}
\end{figure}

\subsubsection{Conditional Adaptive SGN (caSGN)} Now we consider a version of $\left\{\widetilde{\mathbb Q}^t_{{\cal Y}|x}:x\in {\cal X}\right\}$ that varies with the training iteration $t$. We consider training based on mini-batched SGD, where each batch is defined as a random collection of positive examples together with their corresponding negative examples (which are $k$ times many). We take $t$ as the index of a batch.  Suppose training is such that the loss computed for a batch converges and that $\widetilde{\mathbb Q}^t_{{\cal Y}|x}$ converges to $\widetilde{\mathbb P}_{{\cal Y}|x}$ for each $x\in {\cal X}$. Let $T$ be the total number of training iterations (largest batch number). The empirical distribution, say $\widehat{\mathbb Q}^T$, of the noise word-context pair seen during the entire training process is then $\widehat{\mathbb Q}^T(x,y)\!\!=\!\!\sum\limits_{t=1}^T\widetilde{\mathbb Q}^t_{{\cal Y}|x}(y)\widetilde{\mathbb P}_{\cal X}(x)/T$. Under the assumptions stated above, it is easy to see that $\widehat{\mathbb Q}^T$ must converge to $\widetilde{\mathbb P}$ with increasing $T$. Thus, when $T$ is large enough, we can regard
training as a version of mini-batched SGD with ${\mathbb Q}$ chosen as a distribution arbitrarily close to $\widetilde{\mathbb P}$, or a conditional SGN with ${\mathbb Q}_{{\cal Y}|x}$ arbitrarily close to $\widetilde{\mathbb P}_{{\cal Y}|x}$.

This observation motivates us to design the ``Conditional Adaptive SGN'' (caSGN) model. The idea is to parameterize $\widetilde{\mathbb Q}_{{\cal Y}|x}$ using a neural network and force learning with mini-batched SGD 
to make $\widetilde{\mathbb Q}_{{\cal Y}|x}$ converge to $\widetilde{\mathbb P}_{{\cal Y}|x}$. Inspired by GAN \cite
{GoodfellowPMXWOCB14}, we parametrize 
$\widetilde{\mathbb Q}_{{\cal Y}|x}$ 
using a latent variable $Z$ that takes values from a vector space ${\cal Z}$, and 
model $Y$ as being {\em generated} from $(X, Z)$. Figure \ref{fig:generator} (a-c) shows three structures of such a generator. 
Each generator can be implemented as a probabilistic neural network $G$. Then one can formulate the loss function in a way similar to GAN, e.g. in caSGN3 (Figure \ref{fig:c3}), 
\[
\begin{split}
\ell_{\rm caSGN3}=  - 
{\mathbb E}_{x\sim \widetilde{\mathbb P}_{\cal X}} 
\left\{
{\mathbb E}_{y\sim \widetilde{\mathbb P}_{{\cal Y}|x}} \log 
\sigma (s(x, y)) \right.\\
+ \left.{\mathbb E}_{z\sim G_{X|Z}(x),
y\sim G_{Y|XZ}(x, z)} \log 
(-\sigma (s(x, y)))
\right\}
\end{split}
\]
The min-max optimization problem is defined as
\begin{equation}
\label{eq:minMax}
    \left(f^*, g^*, G^*\right):= \arg \min_{f, g} 
    \max_G \ell_{\rm caSGN3} (f, g, G)
\end{equation}
where $G_{Z|X}$ and $G_{Y|XZ}$ are parts of the network $G$.  Following the same derivation as in GAN, the distribution of $(X, Y)$ induced by $G^*$ is $\widetilde{\mathbb Q}$, and optimizing (\ref{eq:minMax}) using mini-batched SGD forces $\widetilde{\mathbb Q}$ converge to $\widetilde{\mathbb P}$.
Note that caSGN1 and caSGN2 in Figure \ref{fig:generator} are special cases of caSGN3, where an additional factorization constraint is applied.

\subsubsection{Unconditional Adaptive SGN (aSGN)} In this model, we simplify the generator so that $Y$ depends only on $Z$ as shown in Figure \ref{fig:u}. Since in every language, the context always depends on the center word, using such a generator, $\widetilde{\mathbb Q}$ tends not to converge to $\widetilde{\mathbb P}$ by construction, except for a very small training sample, to whichthe  model over-fits. 

\subsubsection{ACE} 
ACE \cite{BoseLC18} can also be regarded as a WCC model
with an adaptive noise distribution. In particular, it can be regarded as a special case of caSGN1 or caSGN2, where $Y$ depends only on $X$.

\section{Experiments}\label{experiments}

The main objective of our experiments is to investigate the effect of different forms of noise distribution ${\mathbb Q}$ in the WCC models on the training of embedding and the achieved performance thereby. The approach taken in this study is first to train word embeddings using a given corpus, and then to evaluate the quality of the embeddings using a set of downstream benchmark datasets.

\subsection{Experimental Setup}

\subsubsection{Corpus}
We utilize a Wikipedia dump as our training corpus, following standard text preprocessing procedures including lower-casing and removal of non-lexical items. The processed corpus contains 1.1 billion tokens, with a vocabulary of 153,378 most frequent words. 
Apart from the Wikipedia dump (wiki) used in the main paper, the small corpus text8 \cite{mahoney2011large} is also used, which contains $17$M English words and is pre-processed by lower-casing the text and removing non-lexical items. Words appearing less than 5 times are removed, giving rise to a vocabulary of  $71,290$ unique words. A benefit of adopting this small corpus is to enable repeated runs of training under different random seeds so as to arrive at reproducible results with good confidence.  


\begin{table}[!htbp]
\centering
\caption{Statistics of downstream benchmark datasets}

\setlength{\tabcolsep}{5mm}{

\label{tab:stat}

\begin{tabular}{c|c|c|c}
  \toprule
  \multirow{2}*{Dataset}&\multirow{2}*{Size}&\multicolumn{2}{c}{Covered by}\\
    \cline{3-4}
  && text8 & wiki\\
  \midrule
 WS & 353 &351 &353\\
 SIM & 203 &202 &203\\
 REL & 252 &251 &252\\ 
 RW & 2034 & 951 & 1179\\ 
MTurk287 & 287 &284 &285\\
MTurk771 & 771 &769 &771\\
 MEN & 3000  &2987 &3000\\ 
 RG & 65  &65 &65\\
 MC & 30  &30 &30\\ \midrule
 SimLex-999 & 999  &992 &998\\
  \midrule
  Google Analogy &19544 &17827 &19364\\
  \bottomrule
\end{tabular}
}
\end{table}


\subsubsection{Evaluation Benchmarks}
We employ three types of downstream tasks to comprehensively evaluate embedding quality:

\textbf{Word Similarity}: 
This task is to predict the similarity of a pair of words. A dataset for this task is a collection of records, each containing a word pair and a human-assigned similarity score. When using this task to evaluate a word-embedding model, one computes the cosine similarity for the learned embeddings of the two words in each record. The Spearman correlation coefficient, $\rho$, between the model-predicted similarities and human-assigned scores are then used to evaluate the model. Ten popular data sets used in this study are WS \cite{FinkelsteinGMRSWR02}, WS-SIM, WS-REL \cite{AgirreAHKPS09}; RW \cite{LuongSM13}, MT287 \cite{RadinskyAGM11}, MT771 \cite{halawi2012large}, MEN \cite{BruniBBT12}, RG \cite{RubensteinG65}, MC \cite{miller1991contextual}, and SimLex \cite{hill2015simlex}.


\textbf{Word Analogy}: 
In this task, each record in the dataset is a sequence of four words $(w_1, w_2, w_3, w_4)$ indicating ``$w_1$ is to $w_2$ as $w_3$ is to $w_4$'', and the objective of the task is to predict $w_4$ using the learned embedding $f$. One computes the cosine similarity between $\left( f(w_2)-f(w_1)+f(w_3)\right)$ and the learned embedding of each candidate word in the vocabulary and then selects the word that has the minimum cosine similarity value as $w_4$. The prediction accuracy is the percentage of questions whose $w_4$ is predicted correctly. 
Google’s Analogy dataset, consisting of a Semantic subset and a Syntactic subset \cite{W2V} is used in this task.

\textbf{Named entity recognition (NER)}:
We select the NER task as a real NLP application to compare the WCC models. The CoNLL-2003 English benchmark data set \cite{SangM03} is used for this task, which consists of Reuters newswire articles that are tagged with four entity types (person, location, organization, and miscellaneous). The objective of this task is to identify these entities in the input text. With the learned word embeddings as input, we trained a CNN-BiLSTM-CRF model described in \cite{ma2016end}. 

Note that not all the words in these data sets are included in the vocabulary of the corpus. The number of word pairs in each dataset that are covered by each training corpus is shown in Table \ref{tab:stat}. 

\subsubsection{Compared Models}

\noindent\textbf{SGN} 
We implement the SkipGram model with the default architecture in \cite{MikolovSCCD13}.

\noindent\textbf{aSGN} 
The generator of aSGN is a single multilayer
perceptron (MLP). The input to this generator is a $100$ dimensional latent vector drawn from a standard normal distribution. The generator contains a hidden layer with $512$ dimensions and an output softmax layer that defines the categorical distribution over all candidate context words. There is a ReLU layer between the hidden layer and the output layer.

\noindent\textbf{caSGN1} 
For the first conditional version, for each center word $x$, we concatenate its embedding vector $f(x)$ with a latent vector $z$ drawn from the standard normal distribution as input to the generator. The remaining part uses the same structure as aSGN.

\noindent\textbf{caSGN2} 
Instead of drawing a latent vector independent of $x$ as in caSGN1, we construct a Gaussian distribution whose mean $\mu_x$ and variance $\sigma_x$ both depend on the center word vector $f(x)$. Specifically, we use two Linear-ReLU-Linear structures for $\mu_x$ and $\sigma_x$ respectively and they share the first $512$-dimension linear layer. Then we sample a latent vector from the Gaussian distribution described by $\mu_x$ and $\sigma_x$. The output layer of this generator is again a linear softmax layer.

\noindent\textbf{caSGN3} 
The last version combines the features of caSGN1 and caSGN2. We pick a latent vector from a Gaussian distribution as in caSGN2, and then we concatenate the latent vector with the center word vector before it moves to the next layer as in caSGN1. The remainder of the generator consists of a linear hidden layer and an output layer. Again, there is a non-linear layer ReLU between the hidden layer and the output layer. 

\noindent\textbf{ACE} 
We implement ACE according to \cite{BoseLC18}, except that the NCE \cite{GutmannH12} negative sampler in the model is removed, for the purpose of fair comparison. 

 It is important to clarify that our experimental comparisons focus specifically on analyzing the impact of noise distributions within the WCC framework, rather than pursuing state-of-the-art performance against contemporary large-scale embedding models. While recent transformer-based models \cite{devlin2019bert,tao2024llms,leenv} have demonstrated superior performance through massive pre-training and architectural complexity, such comparisons would be orthogonal to our primary research objective: understanding the fundamental role of noise distribution in embedding learning. Our controlled comparison among WCC variants provides cleaner insights into this specific mechanism, which can inform the design of more sophisticated models in future work.
  
\subsection{Implementation Details}

In this section, we report more details about implementation, including hyperparameter settings and some tricks used in experiments.



\begin{table}[!htbp]
\centering

\caption{Hyper-parameters on text8.}\label{tab:test8}
\setlength{\tabcolsep}{2mm}{
\begin{tabular}{c|c|c}
  \toprule
  \multirow{2}*{Parameter}&\multicolumn{2}{c}{Value}\\
    \cline{2-3}
  & Adaptive & Fixed\\
  \midrule
Learning rate of classifier & 1.0 & 1.0\\
Learning rate of sampler  & 0.05 & -\\
Batch size & 128 & 128 \\
Number of epoch & 20 & 20\\ 
Number of negative sample (k) & 1 & 5 \\ 
$\alpha$ & 20000 & -\\
Latent vector dimension & 100 & -\\
  \bottomrule
\end{tabular}
}

\end{table}

\begin{table}
\caption{Hyper-parameters on wiki.}\label{tab:wiki}
\centering
\setlength{\tabcolsep}{2mm}{
\centering
\begin{tabular}{c|c|c}
  \toprule
  \multirow{2}*{Parameter}&\multicolumn{2}{c}{Value}\\
    \cline{2-3}
  & Adaptive & Fixed\\
  \midrule
Learning rate of classifier & 0.8 & 0.8\\
Learning rate of sampler  & 0.05 & -\\
Batch size & 128 & 128 \\
Number of epoch & 3 & 3\\ 
Number of negative sample (k) & 1 & 5 \\ 
$\alpha$ & 50000 & -\\
Latent vector dimension & 100 & -\\
  \bottomrule
\end{tabular}
}

\end{table}

\subsubsection{Hyper-parameters}
In our experiment, the window size of all the WCC models is 10 so each center token has 10 positive context tokens. We use the subsampling technique to randomly eliminate the words in the corpus following the same scheme in \cite{MikolovSCCD13}. For every WCC model, there are input word embeddings for the center words and output embeddings for the context words. We run all the models trained on text8 for 12 times and report 3-run results for models trained on the Wiki due to the limitation of our computing resource. More details are given in Table.\ref{tab:test8} and Table.\ref{tab:wiki}. Note that the vanilla SGN model in our paper is trained by mini-batched SGD and is implemented via PyTorch, we do not follow the default parameter setting used in the \emph{word2vec} tool.


\subsubsection{Stochastic Node Handling} 

In our adaptive SkipGram model, the generator as an adaptive sampler produces the noise context tokens for the discriminator. The output of the generator is a non-differentiable categorical variable, so the gradient estimation trick is utilized. REINFORCE \cite{Williams92} and the Gumbel-Softmax trick \cite{JangGP17,MaddisonMT17} are two options. Although we conduct some simulations using Gumbel-Softmax, REINFORCE indeed produces slightly better results, so the comparisons of our adaptive SkipGram are all based on the REINFORCE trick. We use the variance reduction technique for the REINFORCE estimator as described in \cite{BoseLC18}, but the result is barely distinguishable from the experiment without this technique, indicating that the high variance problem is not critical here. After we obtain the mean and variance from the center word, we construct a Gaussian distribution and sample the noise from this distribution. This sampling layer is also a non-differentiable stochastic node, so we apply the reparameterization trick ~\cite{KingmaW13,RezendeMW14} to this layer.

\subsubsection{Entropy Control} 
When the generator finds a specific context word that receives a relatively high score from the discriminator for many center words, it tends to become 'lazy' and not explore other candidates. Thus, the entropy of the categorical distribution given by the generator is getting smaller and smaller during training. In this case, the discriminator cannot learn more about the true data structure, and the binary classification task is no longer challenging. To guarantee the rich diversity of tokens produced by the generator, we apply a regularizer to give the generator a high penalty when the entropy of the categorical distribution is small. The regularizer proposed by \cite{BoseLC18} uses the entropy of a prior uniform distribution as a threshold and encourages the generator to have more candidates:
\[
    \label{eq:regularizer}
     R(x) = \max(0, \log(\alpha)-H(y|x)), 
\]
where $\log(\alpha)$ is the entropy of the uniform distribution and $H(y|x)$ is the entropy of the categorical distribution defined by the generator. Indeed, the entropy control trick has already been used in previous work. For example, the prior distribution \cite{MikolovSCCD13} used is ``3/4-unigram'', and the entropy of this distribution is higher than the unigram distribution.

\subsection{Main Results}
We first show the performance of 100-dimensional word embeddings learned by vanilla SGN, ACE, and 4 our own adaptive SGN models in downstream tasks. Then we discuss the impact of the different ${\mathbb Q}$ discussed before. Notice that the results presented here are not competitive to the current state-of-the-art results, which not only need a sufficiently larger corpus and embedding dimension but a more complex structure of $f$ \footnote{In Section 3.2, $f$ is only taken by the simplest choice, but indeed it can be more complicated (e.g., Transformer).} is also required.

\begin{table*}[!htbp]
\caption{Spearman's $\rho$ ($\ast 100$) on the Word Similarity tasks.}
 \label{tab: WS-Large}
\centering
\setlength{\tabcolsep}{3mm}{
\begin{tabular}{c|c|c|c|c|c|c|c|c|c|c}
  \toprule
 {\bf Model}&  WS &  SIM & REL&MT287& MT771&RW&MEN&MC&RG&SimLex\\
    \midrule
  SGN&69.56&75.23&64.03&63.67&59.83&40.26 &69.71&64.47&70.99&30.32\\
  ACE &73.15&76.82&{\bf 69.88}&68.10&60.80&40.48 &71.08&{\bf 78.11}&77.42&30.08\\
  \midrule
  aSGN & 70.11 & 74.30 &65.08 & {\bf 69.16} & 61.72 & 41.95 & 71.09 & 68.92&69.71&32.25\\
  caSGN1 &  72.02 & 77.01 & 66.68 & 67.66 &60.36& {\bf 42.05} & 71.37 &   75.31&{\bf 77.97}&31.75\\
  caSGN2 & {\bf 73.95} & 78.27 & 69.81 & 64.86& 62.60& 41.92  &  {\bf 73.33} & 72.39&71.79&{\bf 34.01}\\
  caSGN3 & 73.80 & {\bf 78.52} & 69.17 & 65.97&{\bf 63.31}& 41.21 & 72.47 & 74.33&72.42&32.42\\
  \bottomrule
\end{tabular}
}
\end{table*}
\subsubsection{Performance on Downstream Tasks}

Table \ref{tab: WS-Large}  and \ref{tab:wa-large} show the models' performances on Word Similarity and Word Analogy, respectively. We see that each adaptive SGN model outperforms vanilla SGN on all data sets. In particular, we note that caSGN2 is able to give the highest score on Simlex, where SGN often achieves poor scores. Table \ref{tab:ner} 
compares the validation and test score $F_1$ for NER. Clearly, the improvement over SGN is statistically significant for most of our models.

\begin{table}[!htbp]
\caption{Accuracy on the Word Analogy task.}\label{tab:wa-large}
\centering
\setlength{\tabcolsep}{5mm}
{
\begin{tabular}{c|c|c|c}
  \toprule
  {\bf Model}&Semantic&Syntactic&Total\\
  \midrule
SGN & 55.03 &46.97 &50.64\\
ACE & 56.13 & 48.06&  51.74\\
\midrule
aSGN & 54.73& {\bf 49.93}& 52.11\\
caSGN1 & 55.76& 49.60& 52.40\\
caSGN2 & {\bf 58.84}& 48.94& {\bf 53.45}\\
caSGN3 & 58.40& 48.98& 53.27\\
  \bottomrule
\end{tabular}
}
\end{table}

Among the adaptive samplers (caSGN1,2,3, aSGN and ACE), caSGN3 seems to be able to capture more information than the others, but we find that caSGN2 has superior results in practice. Unfortunately, we do not have a formal justification to address this, and we leave the comparison between these models for future work. 

In addition, we notice that ACE also performs well on some tasks, but this fact should not be taken negatively on our results. As stated above, ACE is essentially an adaptive conditional member of the WCC. Its good performance further endorses some claims of our paper, namely WCC is a useful generalization and adaptive conditional schemes can be advantageous.

\begin{table}[!htbp]
\caption{Validation set and test set $F_1$ score on NER.Statistical significance difference to the baseline SGN using t-test: $\ddag$ indicates $p$-value $<  0.05$ and $\dag$ indicates $p$-value $<  0.01$.}\label{tab:ner}
\centering
\setlength{\tabcolsep}{7mm}
{
\begin{tabular}{c|l|l}
  \toprule
  {\bf Model}&Val.&Test\\
  \midrule
SGN & 93.67 &88.38\\
ACE &93.85 & 88.84$^\dag$\\
\midrule
aSGN & 93.71 & 88.62\\
caSGN1 & 93.79& 88.93$^\ddag$\\ 
caSGN2 & {\bf 94.01}$^\dag$& {\bf 89.04}$^\dag$\\ 
caSGN3 & 93.93$^\dag$& 88.88$^\ddag$\\
  \bottomrule
\end{tabular}
}
\end{table}

\subsubsection{Noise Distribution Impact Analysis}\label{sec:432}
To demonstrate Hypothesis \ref{hp:QequalP}, we show the estimated Jensen–Shannon divergence (JSD) between ${\mathbb Q}$ and ${\mathbb P}$ in Table \ref{tab:jsd}. Concretely, the objective function of the generator in GAN is taken to estimate JSD \cite{GoodfellowPMXWOCB14}. For comparison at the same scaling level, JSD of different ${\mathbb Q}$ is computed using the same embeddings learned by caSGN2. Notice that we are not able to fairly compare JSD of different adaptive ${\mathbb Q}$ in the same way, since the generators in these models are jointly trained with word embeddings, and using the same word embeddings will not enable different generators to produce valid samples. In addition, Figure \ref{fig:curve} shows the performance-varying curves on two datasets. 

\begin{table}[!htbp]
\caption{Estimated JSD between ${\mathbb Q}$ and ${\mathbb P}$.}\label{tab:jsd}
\centering
\setlength{\tabcolsep}{7mm}
{
\begin{tabular}{c|c}
 \toprule
  $\widetilde{\mathbb Q}$&Estimated JSD\\
  \midrule
ufSGN & 2.35\\
ugSGN & 0.71\\
3/4-SGN & 1.25\\ 
\midrule
caSGN2 & 1.02\\
  \bottomrule
\end{tabular}

}
\end{table}

In Table \ref{tab:jsd}, we find that the ``unigram ${\mathbb Q}$" is closest to ${\mathbb P}$ in the JSD sense but gives poor performance as shown in Figure \ref{fig:curve}. It is easy to see that ``unigram" is ``sharper'' than ``3/4-unigram", or in other words, ``unigram" has a lower entropy, which indicates the classifier of ugSGN is trained by limited frequent words in corpus most of the time. It turns out that its learned embeddings 'over-fit' those frequent words but 'under-fit' others. This does not violate Hypothesis \ref{hp:QequalP}, as discussed in Remark 2, ${\mathbb Q}$ that close to ${\mathbb P}$ cannot perform well due to the mismatch of the downstream task and the training corpus.

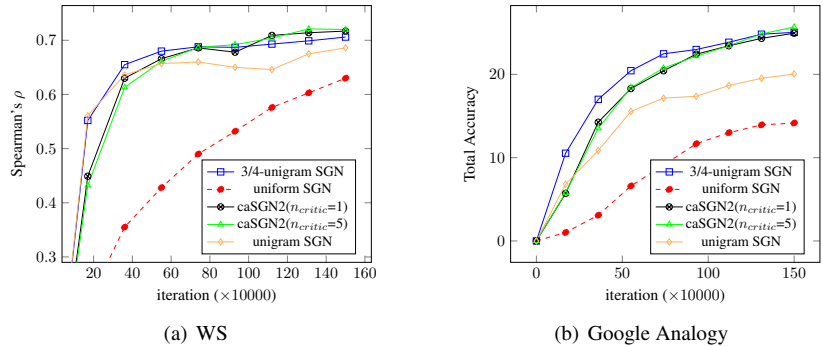
\begin{figure}[ht!]
\centering
\subfigure [{WS}] {\label{fig:aa}
\scalebox{0.48}{\input{plot_ws.tex}}
}
\qquad
\hspace{-.9cm}
\subfigure [{Google Analogy}] {\label{fig:bb}
\scalebox{0.48}{\input{plot_rw.tex}}
}
\caption{Part of performance varying curves on WS and Google Analogy.}
\label{fig:curve}
\end{figure}

In addition, Figure \ref{fig:curve} shows that ufSGN performs poorly. Notably, ufSGN's ${\mathbb Q}$ is far from ${\mathbb P}$ as shown in Table \ref{tab:jsd} so the classification task is ``not challenging'' and the classifier does not need to learn much about the data. Notice that ufSGN's performance converges very slowly, and this phenomenon is consistent with Remark 1. 

Figure \ref{fig:curve} also presents two caSGN2 models with different $n_{critic}$\footnote{Notation $n_{critic}$ is the number of
iterations to apply to the discriminator before per generator iteration.}. We notice that caSGN2 with $n_{critic}$=1 does not converge faster than $n_{critic}$=5. Thus, we can choose a larger $n_{cirtic}$ to reduce the running time (see Section \ref{sec:furtherdiscussion} for a further discussion of time complexity). At the beginning of training, the performance of caSGN2 improves slightly slower than 3/4-ugSGN and ugSGN. This is not surprising since caSGN2's ${\mathbb Q}$ is not ``competitive'' at first. After getting closer to ${\mathbb P}$ (as indicated in Table \ref{tab:jsd}), caSGN2 outperforms other models. These observations convey the message that the best ${\mathbb Q}$ is ${\mathbb P}$ and GAN is a desired method for applying hypothesis \ref{hp:QequalP} to practice because word pairs can be uniformly trained at the initialization and the generator will force the classifier to learn as much as possible about the data when ${\mathbb Q}$
gradually moves to ${\mathbb P}$.

\begin{table*}[!htbp]
\caption{Spearman's $\rho$ ($\ast 100$) on the word similarity tasks (text8). Statistical significance difference from baseline SGN using the t test: $\dag$ indicates the $p$ -value $<  0.05$ and $\ddag$ indicates the $p$ -value $<  0.01$ } \label{tab: WS}
\centering
\setlength{\tabcolsep}{3mm}{
\begin{tabular}{c|c|c|c|c|c|c|c|c|c|c}
  \toprule
  Models&  WS &  SIM & REL&MT287& MT771&RW&MEN&MC&RG&SimLex\\
    \midrule
  SGN&70.58&74.54&68.10&64.29&55.59&36.63 &62.16&60.82&60.17&29.69\\
  ACE &71.49$^\ddag$&74.61&69.50$^\ddag$&65.52$^\ddag$&56.63$^\ddag$&{\bf 37.85}$^\ddag$ &62.75&62.65&62.39&30.37$^\ddag$\\
  aSGN & 71.12$^\dag$ & 74.76 & 68.82 & {\bf 65.67}$^\ddag$ &56.47$^\dag$& 37.58$^\ddag$ & 62.63 & 62.36&62.36&30.49$^\ddag$\\
  caSGN1 & 71.72$^\ddag$ & {\bf 75.11} & {\bf 69.77}$^\ddag$ & 65.63$^\ddag$ &56.63$^\ddag$& 37.63$^\ddag$ & {\bf 63.40}$^\dag$ &  62.54$^\dag$&64.18$^\ddag$&30.36$^\ddag$\\
  caSGN2 & {\bf 72.02}$^\ddag$ & 75.05 & 69.64$^\ddag$ & 65.44$^\ddag$& {\bf 57.02}$^\ddag$& 37.61$^\ddag$  &  63.36$^\dag$ & {\bf 62.86}$^\dag$&{\bf 64.63}$^\ddag$&{\bf 30.79}$^\ddag$\\
  caSGN3 & 71.74$^\ddag$ & 74.61 & 69.63$^\ddag$ & 65.57$^\ddag$&56.56$^\ddag$& 37.78$^\ddag$ & 62.69 & 62.61&62.52&30.31$^\ddag$\\
  \bottomrule
\end{tabular}
}
\end{table*}

\begin{table}[!htbp]
\caption{Accuracy on the word analogy task (text8).}\label{tab:wordanalogy}
\centering
\setlength{\tabcolsep}{5mm}{
\begin{tabular}{c|c|c|c}
   \toprule
  Model&Semantic&Syntactic&Total\\
   \midrule
SGN & 20.50 &26.77 &24.16\\
ACE & 20.43 &28.25$^\ddag$ &25.00$^\dag$\\
aSGN & 20.84 &27.86$^\ddag$ &24.94$^\dag$\\
caSGN1 & 21.25 &{\bf 28.30}$^\ddag$ &25.36$^\ddag$\\ 
caSGN2 & {\bf 21.99} & 27.85$^\ddag$ & {\bf 25.41}$^\dag$\\ 
caSGN3 & 21.03 &27.87$^\ddag$ &25.03\\
   \bottomrule
\end{tabular}
}
\end{table}

\subsection{Robustness and Statistical Analysis}
In this section, we provide more results including experimental results of different word dimensions and violin plots of models trained on text8. Furthermore, the extensive experimental results verify that these observations are robust across different corpus sizes and word dimensions. In particular, we provide the results of the significance test based on 12 runs, and their corresponding violin plots are also given.

\begin{figure*}[t]
\centering
\subfigure [WS] {\label{fig:wsa}
\includegraphics[width=0.17\textwidth]{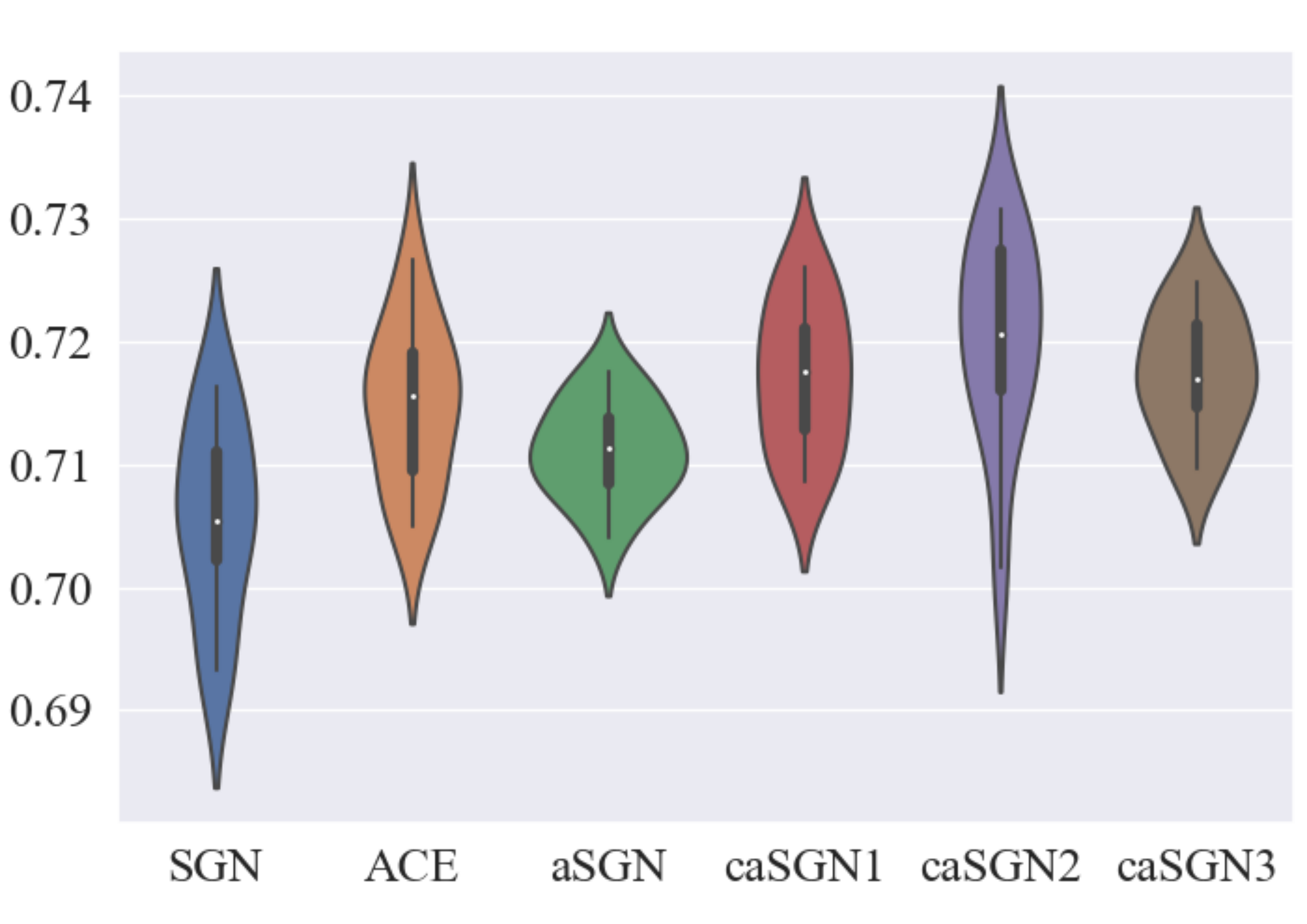}}
 \quad
\subfigure [SIM] {\label{fig:wsb}
\includegraphics[width=0.17\textwidth]{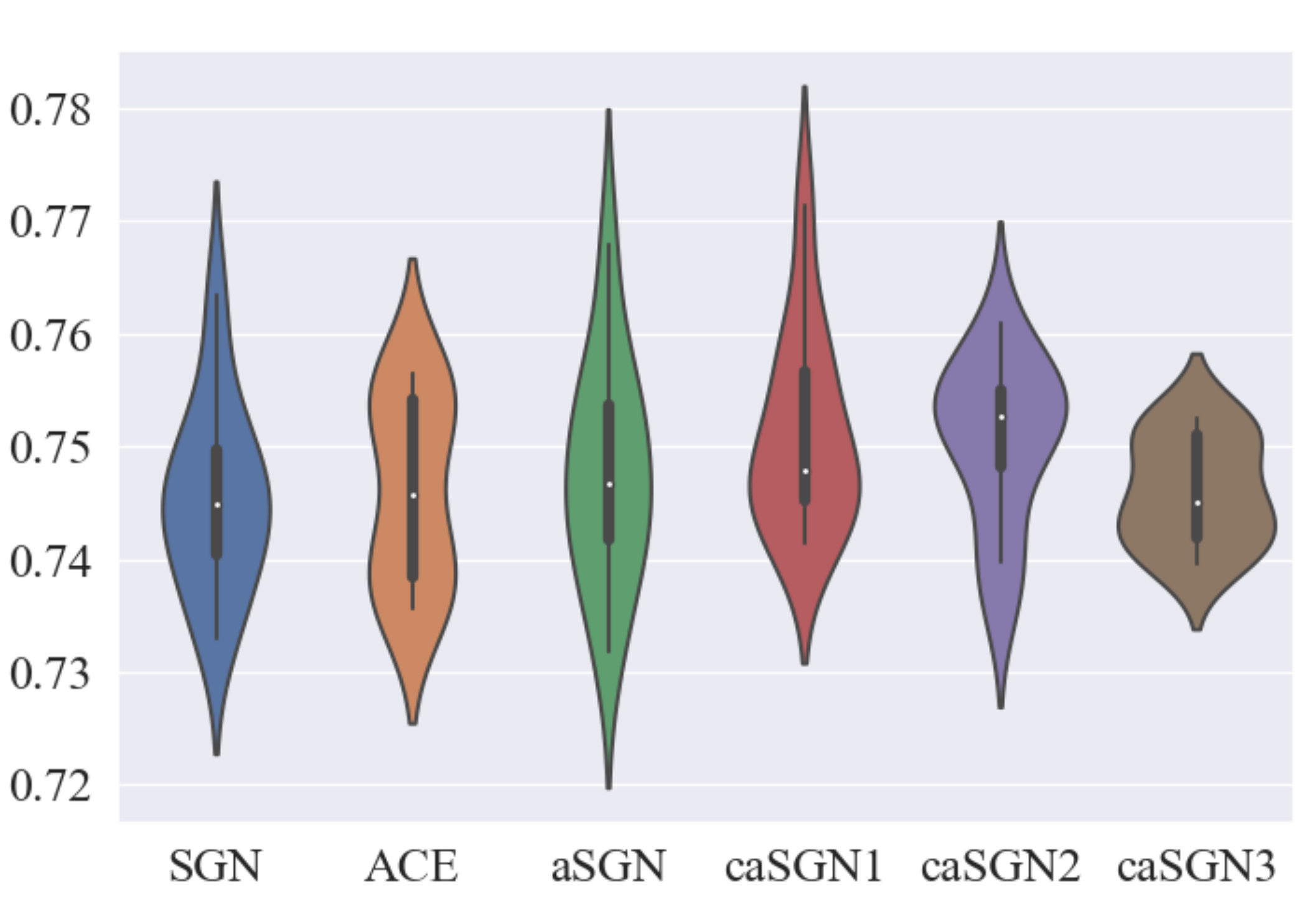}}
 \quad
\subfigure [REL] {\label{fig:wsc}
\includegraphics[width=0.17\textwidth]{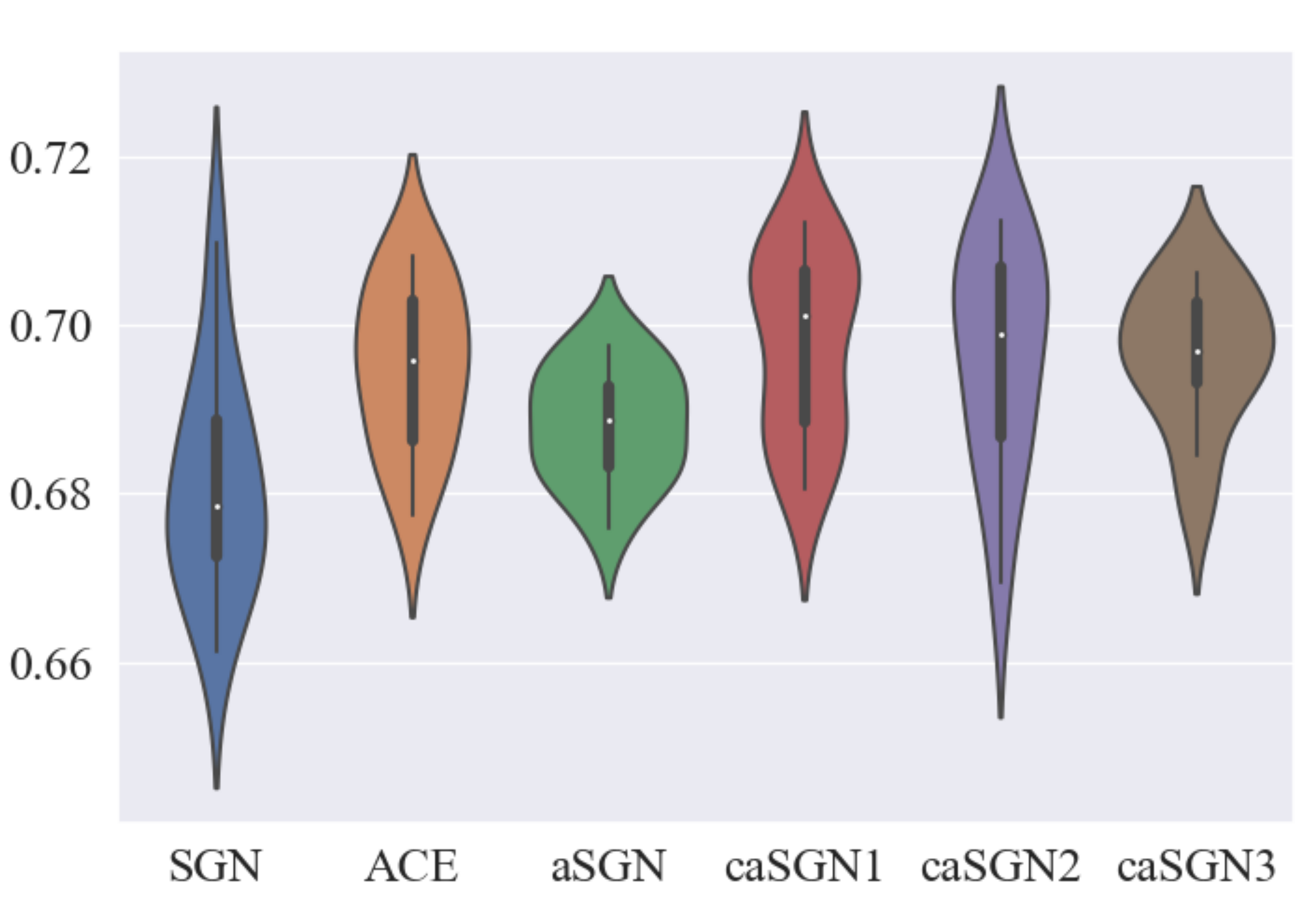}}
 \quad
\subfigure [MTurk287] {\label{fig:wsd}
\includegraphics[width=0.17\textwidth]{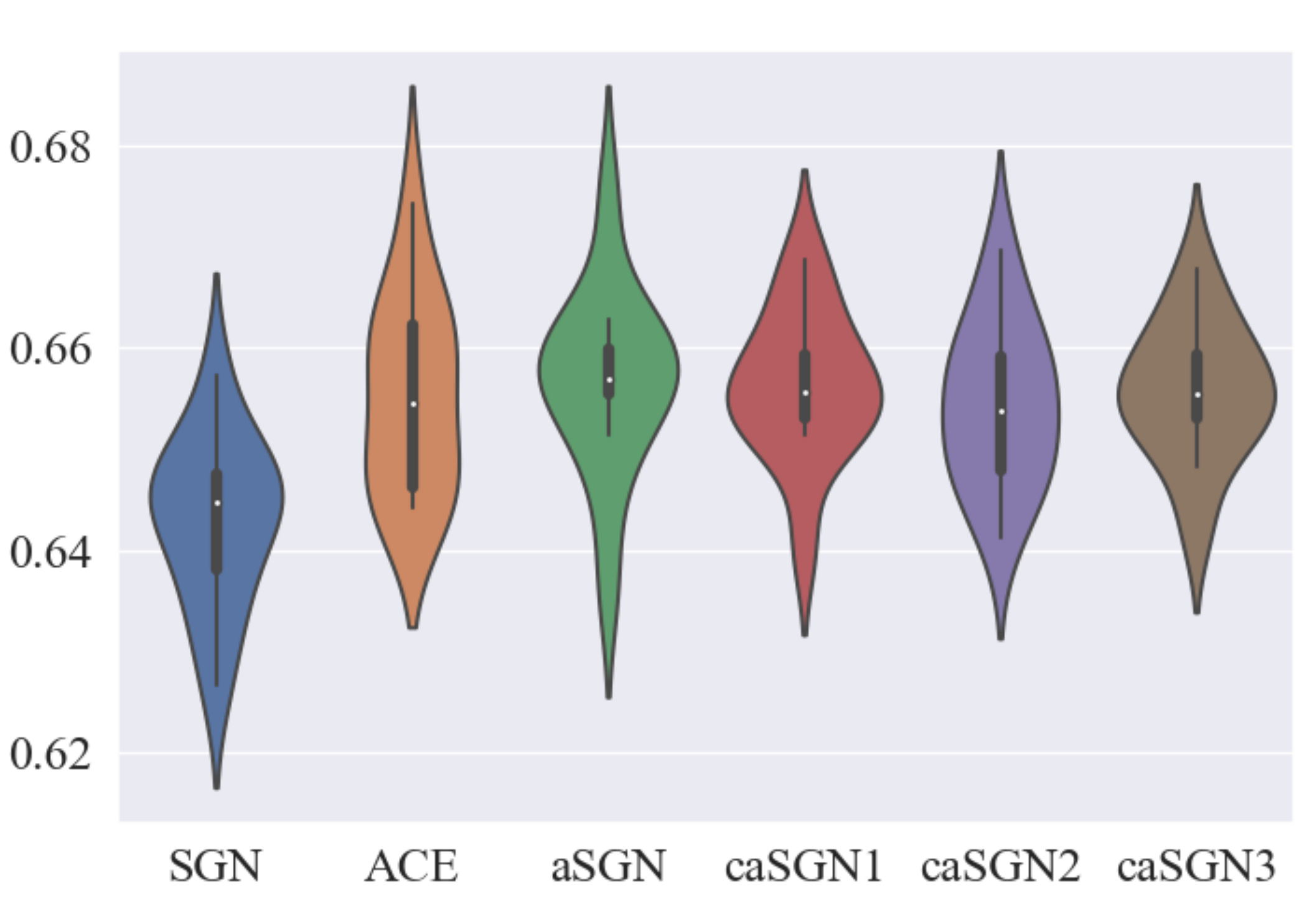}}
 \quad
\subfigure [MTurk771] {\label{fig:wse}
\includegraphics[width=0.17\textwidth]{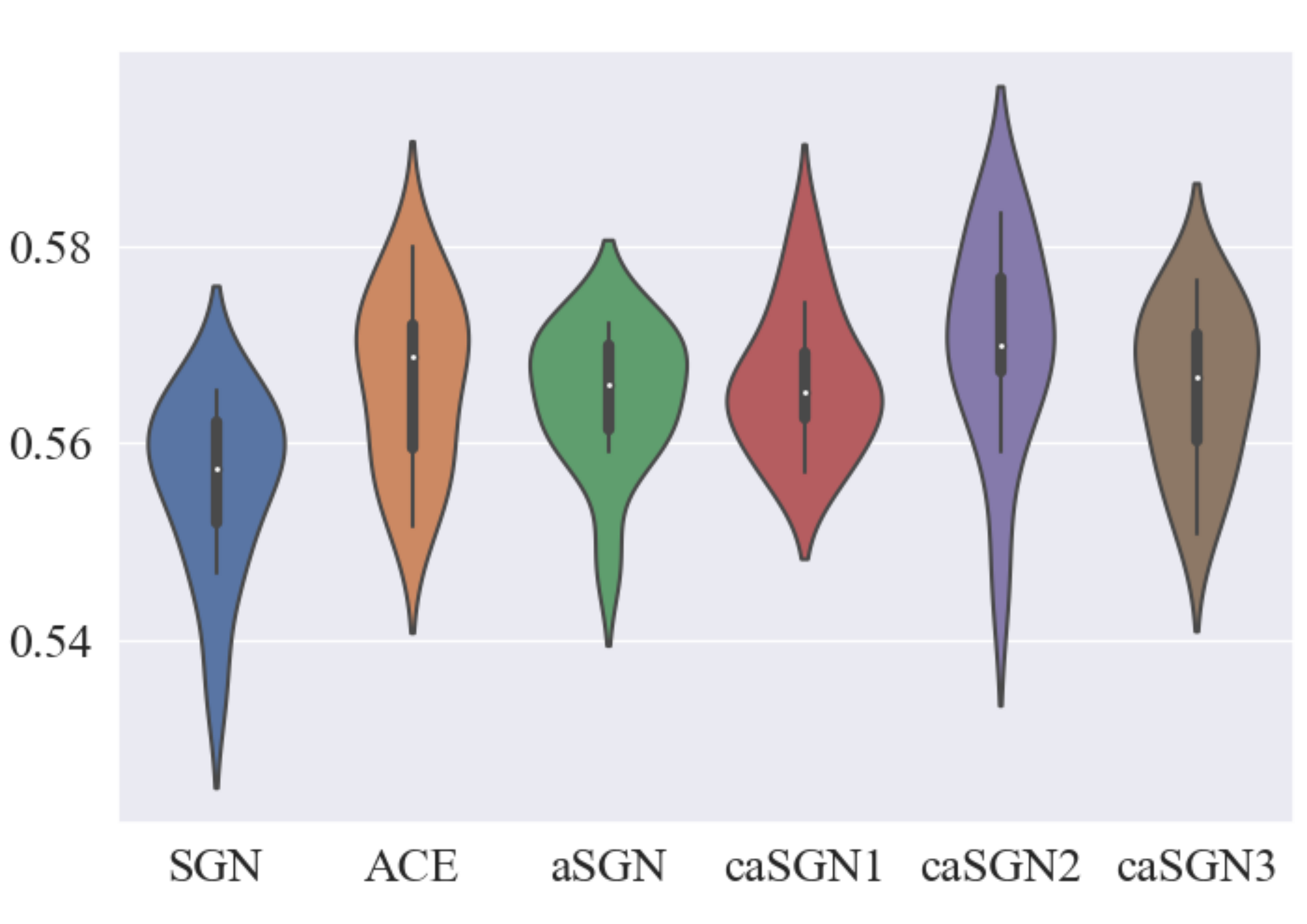}}
 \quad
\subfigure [RW] {\label{fig:wsf}
\includegraphics[width=0.17\textwidth]{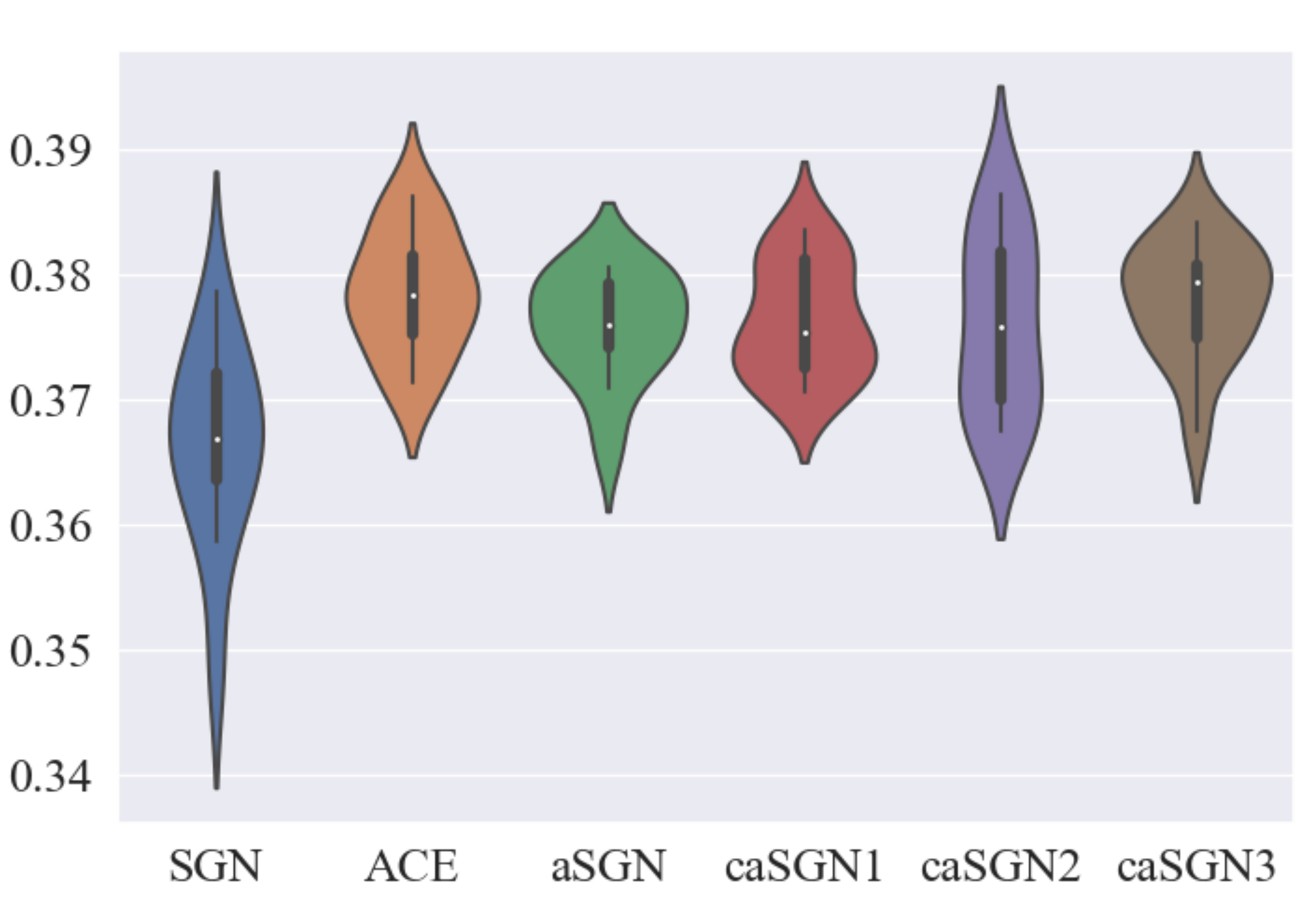}}
 \quad
\subfigure [MEN] {\label{fig:wsg}
\includegraphics[width=0.17\textwidth]{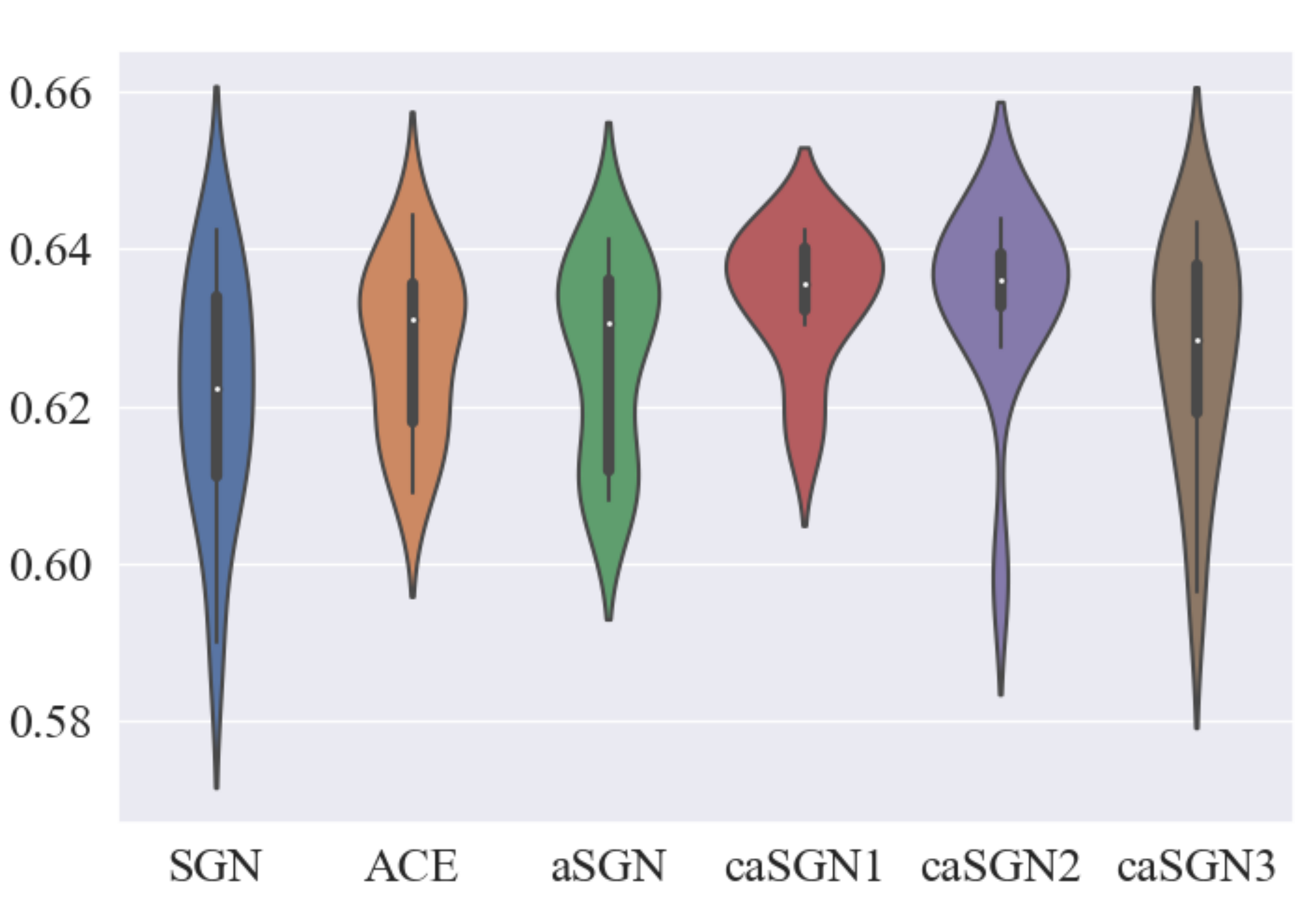}}
 \quad
\subfigure [MC] {\label{fig:wsh}
\includegraphics[width=0.17\textwidth]{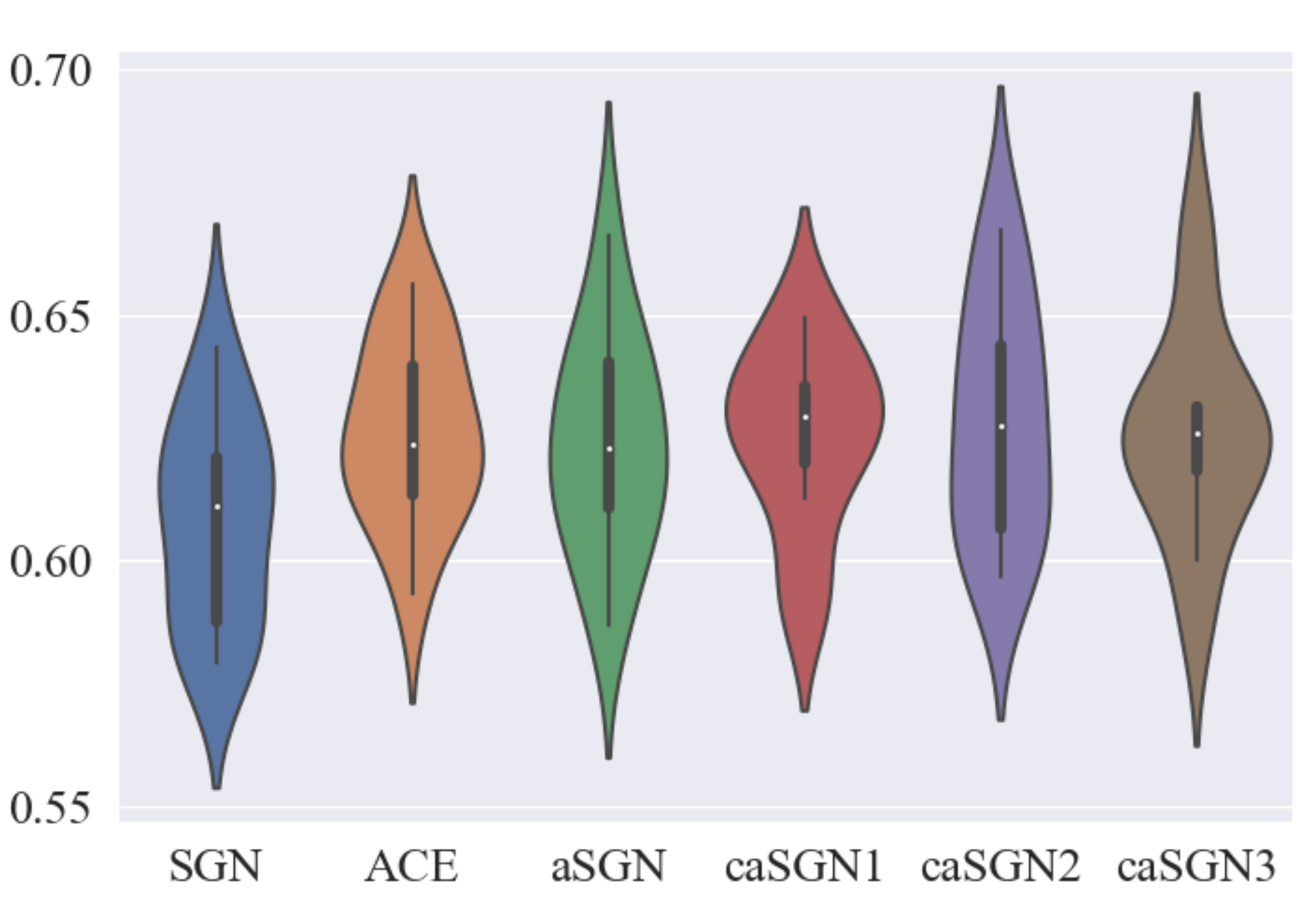}}
 \quad
\subfigure [RG] {\label{fig:wsi}
\includegraphics[width=0.17\textwidth]{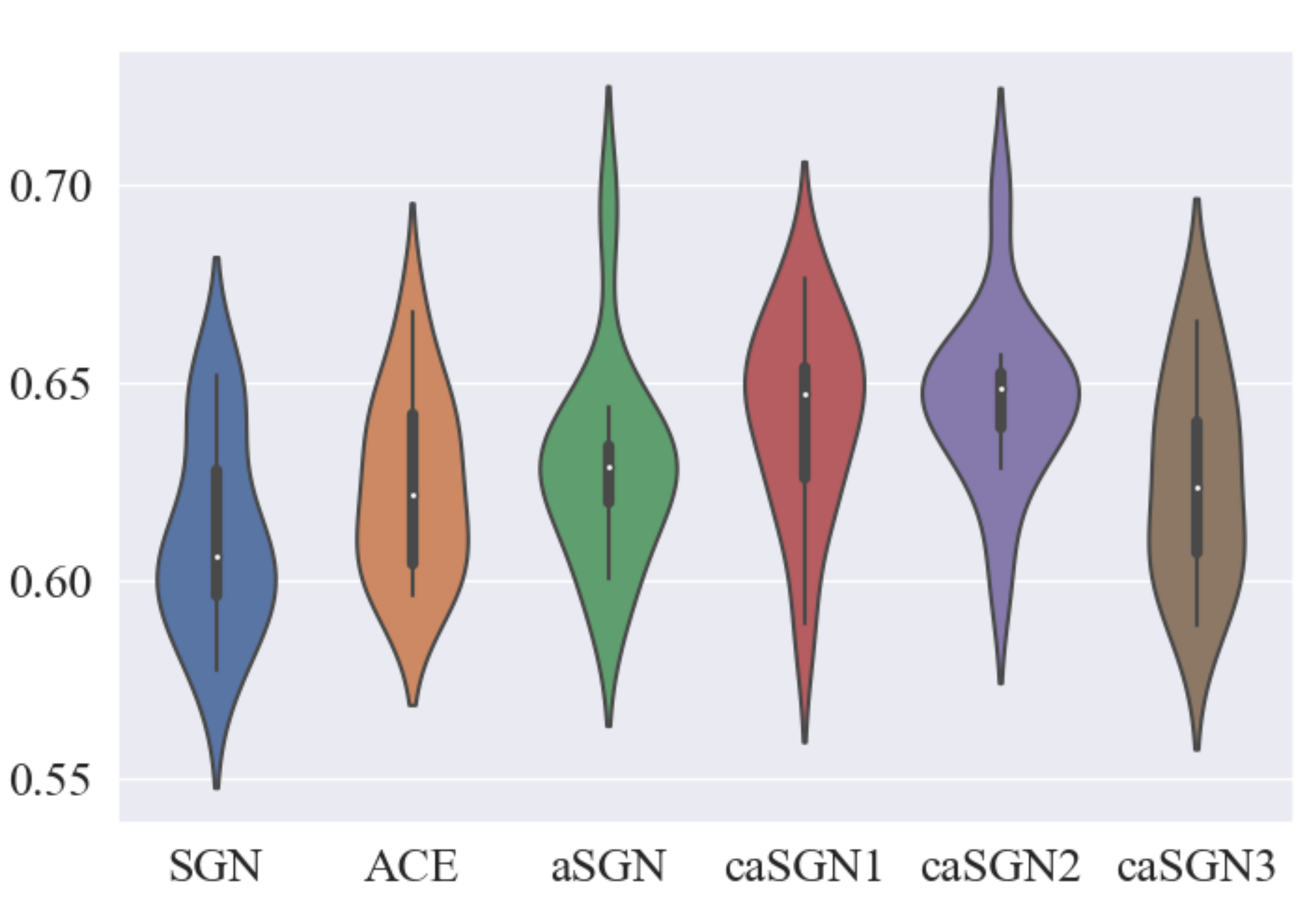}}
 \quad
\subfigure [SimLex] {\label{fig:wsj}
\includegraphics[width=0.17\textwidth]{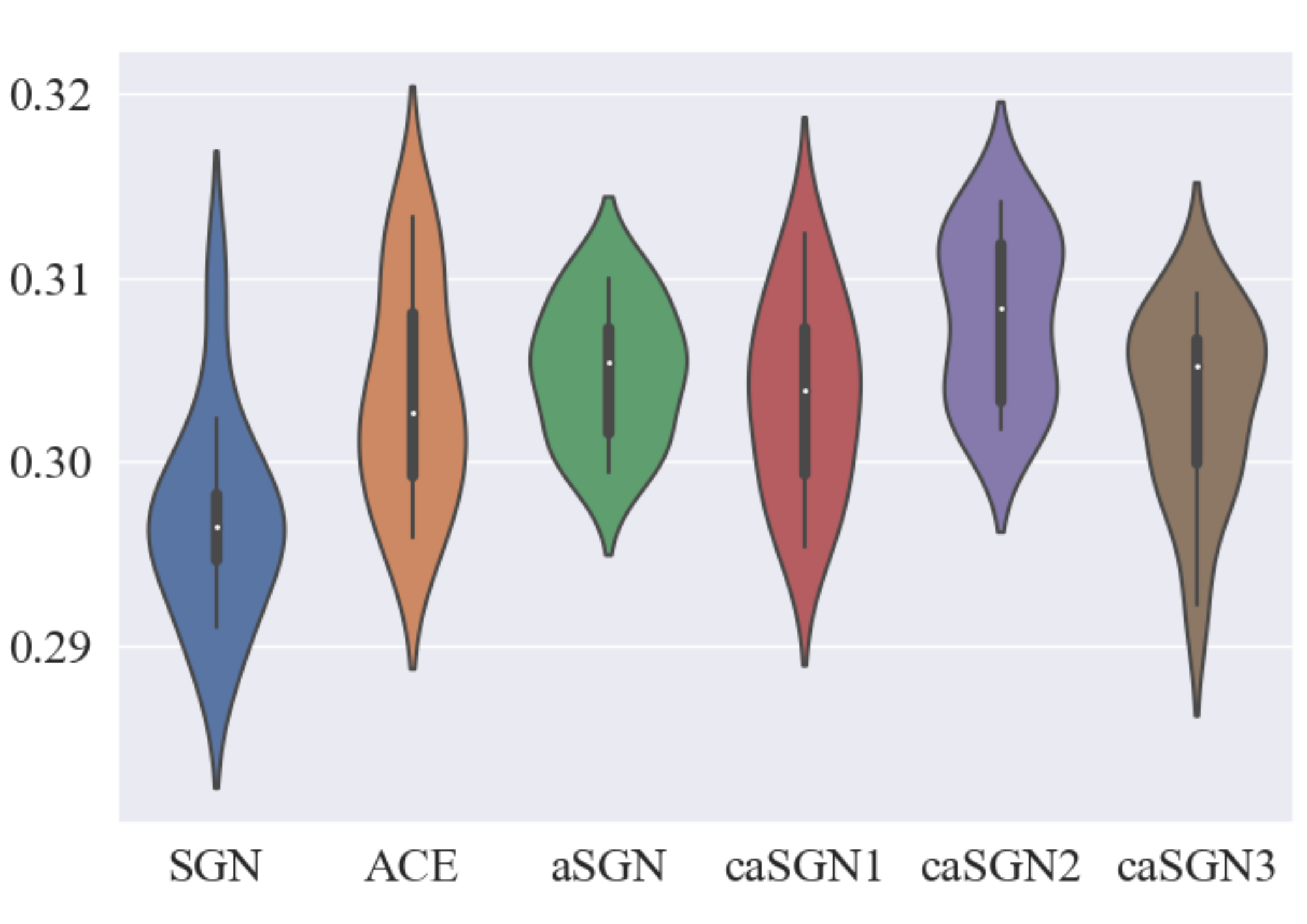}}
 \quad
\caption{Violin Plots of the Word Similarity Task.}
\label{fig:violinWS}
\end{figure*}

\begin{figure}[t]
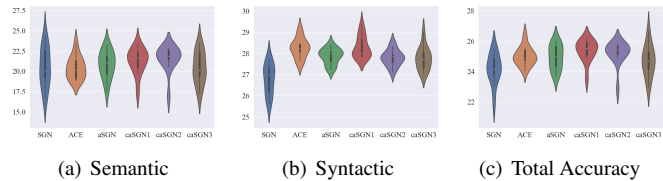

\centering
\subfigure [Semantic] {\label{fig:waa}
\includegraphics[width=0.145\textwidth]{Semantic.pdf} }
\subfigure [Syntactic] {\label{fig:wab}
\includegraphics[width=0.145\textwidth]{Syntactic.pdf} 
}
\subfigure [Total Accuracy] {\label{fig:wac}
\includegraphics[width=0.145\textwidth]{Total.pdf}}
\caption{Violin Plots of the Word Analogy Task.}
\label{fig:violinWA}
\end{figure}

\begin{table*}[h]
\caption{Spearman's $\rho$ ($\ast 100$) of 50-dimension word embeddings on the word similarity tasks (text8).}
 \label{tab: WS-50}
\centering
\setlength{\tabcolsep}{3mm}{
\begin{tabular}{c|c|c|c|c|c|c|c|c|c|c}
   \toprule
 Models&  WS &  SIM & REL&MT287& MT771&RW&MEN&MC&RG&SimLex\\
    \midrule
  SGN&68.07&71.90&65.40&63.77&54.36&35.53 &60.64&54.62&56.03&26.79\\
  ACE &69.03&72.24& 65.87&64.72&{\bf 55.50}&36.41 &{\bf 61.71}& 58.02&56.71&{\bf 29.87}\\
  aSGN & 68.66 & 71.80 &65.80 &  65.27 & 54.44 & 35.55 & 60.74 & 52.17&54.68&26.90\\ 
  caSGN1 &  68.83 & 71.67 & 65.36 & 64.32 &53.76&  {\bf 36.85} & 60.11 & 57.34& 54.89 &28.15\\
  caSGN2 &  {\bf 70.10} & {\bf 73.23} & {\bf 67.10} & {\bf 65.72}& 55.07& 36.75  &  61.04 & 57.60&57.93& 28.73\\
  caSGN3 & 69.63 & 73.07 & 66.80 & 64.79&54.19 & 36.25 &60.70 & {\bf 59.21} & {\bf 58.42}&26.96\\
   \bottomrule
\end{tabular}
}
\end{table*}

\begin{table*}[h]
\caption{Spearman's $\rho$ ($\ast 100$) of 200-dimension word embeddings on the word similarity tasks (text8).}
 \label{tab: WS-200}
\centering
\setlength{\tabcolsep}{3mm}{
\begin{tabular}{c|c|c|c|c|c|c|c|c|c|c}
   \toprule
 Models&  WS &  SIM & REL&MT287& MT771&RW&MEN&MC&RG&SimLex\\
    \midrule
  SGN&70.94&74.41&68.17&63.43&56.42&37.87 &62.33&66.14&61.92&31.03\\
  ACE &71.76&74.84& 68.72& {\bf 65.52}&57.80&{\bf 40.11} &63.75& {\bf 70.44}&68.19&32.14\\
  aSGN & 71.96 & {\bf 74.95} &69.50 &64.61 & {\bf 58.19} & 39.02 & 63.73 & 66.03&64.80&32.66\\
  caSGN1 &  69.16 & 71.72 & 67.99 & 62.52 &57.50&  38.48 & 65.61 & 70.26& 66.98 &31.37\\
  caSGN2 &  {\bf 72.23} & 74.80 & 69.67 & 65.09& 58.01& 38.63  &  64.75 & 68.43&68.02& 32.86\\
  caSGN3 & 71.44 & 72.46 & {\bf 71.28} & 65.33&57.54 & 39.84 &{\bf 66.42} &  64.85 & {\bf 69.14}&{\bf 32.88}\\
  \bottomrule
\end{tabular}
}
\end{table*}

\begin{table}[!ht]
\caption{Accuracy of 50-dimension word embeddings on the word analogy task (text8).}\label{tab:wa-50}
\centering
\setlength{\tabcolsep}{5mm}{
\begin{tabular}{c|c|c|c}
  \toprule
  Model&Semantic&Syntactic&Total\\
  \midrule
SGN &15.47 & 20.83 &18.60 \\
ACE & 17.89 &20.56 &19.45\\
aSGN & 16.83 &{\bf 21.46} &19.53\\
caSGN1 & {\bf 17.99} & 20.09 & 19.22\\ 
caSGN2 &  17.53 & 21.15 & {\bf 19.64}\\ 
caSGN3 &  17.91 &20.79 &19.59\\
   \bottomrule
\end{tabular}
}
\end{table}

\begin{table}[h]
\caption{Accuracy of 200-dimension word embeddings on the word analogy task.}\label{tab:wa-200}
\centering
\setlength{\tabcolsep}{5mm}{
\begin{tabular}{c|c|c|c}
  \toprule
  Model&Semantic&Syntactic&Total\\
  \midrule
SGN &28.56 &26.12 &27.13 \\
ACE &  28.70 &27.56 &28.03 \\
aSGN &  27.10& 27.94& 27.59\\
caSGN1 &27.63 &{\bf 28.47} &28.12\\ 
caSGN2 & 31.59& 27.23& {\bf 29.05}\\ 
caSGN3 & {\bf 32.67} &28.15 &29.02\\
  \bottomrule
\end{tabular}
}
\end{table}

\subsubsection{Statistical Significance}

The results for the 100-dimensional word embeddings and the statistical significance tests, based on 12 runs on the text8 corpus, are provided in Table \ref{tab: WS} and Table \ref{tab:wordanalogy}.  In Table \ref{tab: WS}, we see that each adaptive SGN model outperforms the vanilla SGN on all data sets. Based on the 12 running results, the violin plots are shown in Figure \ref{fig:violinWS}. Basically, most of the differences in Table \ref{tab: WS} are significant, except for SIM for which no difference is significant ($p$-value $>0.05$).  As shown in Table \ref{tab:wordanalogy}, adaptive WCC models, particularly conditional adaptive models, continue to perform better than other models. Specifically, there is no statistically significant difference in semantics, but the differences in total accuracy are significant, except for caSGN3. You can refer to Figure \ref{fig:violinWA} for the violin plots.

In summary, based on 12 independent runs on the text8 corpus, the performance advantages of adaptive models are statistically significant ($p < 0.05$) for most datasets and tasks. Violin plots in Figures \ref{fig:violinWS} and \ref{fig:violinWA} visually confirm the consistent performance distributions.

\subsubsection{Impact of Embedding Dimension}
We show the results of the 50-dimension word embeddings in Table \ref{tab: WS-50} and Table \ref{tab:wa-50}, and the results of 200-dimensional word embeddings in Table \ref{tab: WS-200} and Table \ref{tab:wa-200}. We can see that the experimental results verify that the observations are robust across various word dimensions.

In summary, the results demonstrate that our observations remain consistent across embedding dimensions of 50, 100, and 200, confirming the robustness of noise distribution effects across model capacities


\section{Further discussions}\label{sec:furtherdiscussion}
\subsection{Limitation: time complexity}
Arguably, one of the reasons for using NCE \cite{GutmannH12} is the efficiency gain in terms of time complexity. As each adaptive SGN becomes complex due to the training of the generator, it would likely require a growing amount of time. In fact, compared to vanilla SGN, the running time of adaptive SGN models has increased more than tenfold. 

However, it is important to note that word embeddings, or even other embeddings, are usually pre-trained for downstream tasks. Thus, the computation is a one-off cost. Moreover, there are many specific tricks that can accelerate the training of GAN. In addition to the way mentioned in Section~\ref{sec:432}, one can see \cite{BoseLC18} and \cite{BudhkarVHR19} for more suggestions.

\subsection{WCC vs NCE}
While NLP researchers tend to regard SGN as an application of NCE to word embedding, we assert that this understanding requires clarification. Although both SGN and NCE use noise samples to construct a binary classification problem, and it is possible to convert NCE to a conditional version so that the resulting binary classification problem looks the same as that in SGN, the two approaches differ fundamentally in their theoretical foundations.

NCE aims at learning a possibly unnormalized ``density function'' $f$ on a space of examples, through a set of observed examples; the starting point of NCE is the setting up of the function $f$. NCE then draws noise examples to form a negative class and reformulates the learning problem as a binary classification problem. It is remarkable that the parametrization of $f$ is independent of any noise distribution used in NCE.

SGN aims at learning the representation of elements in a space (where each element is a word-context pair); the starting point of SGN is a binary classification formulation. That is, in SGN, there does not exist a parametrized density function independent of the choice of noise distribution. If one must equate SGN with a special case of NCE, then the effective density model in SGN would have to be parametrized by the noise distribution.

This difference between NCE and SGN also results in their differences in training: in NCE, one must be able to evaluate the noise distribution, but this is not required in SGN; in NCE, the partition function of the unnormalized density function f must also be estimated during training, but this is also not required in SGN. Thus, although SGN is inspired by NCE, it is not NCE.

The distinction of the two extends to distinguishing WCC from NCE. Specifically, WCC is a generalization of SGN to allow for more general forms of noise distribution. This generalization is ``orthogonal'' to the difference between NCE and SGN. 
Although the formalism of NCE allows any noise distribution, negative sampling with an arbitrarily distributed noise is established for the first time in this paper. When generalizing SGN to WCC, the PMI-MF result no longer holds. Another contribution of this paper is establishing the ``correctness’’ of WCC in Corollary \ref{cor:justifyWCC}.

\subsection{WCC vs ACE}

Our WCC framework generalizes SGN to arbitrary noise distributions while maintaining theoretical soundness, as established in Corollary \ref{cor:justifyWCC}. Within this broader framework, ACE may be viewed as an adaptive conditional member of WCC, with a structure corresponding to Figure~\ref{fig:generator} (a) or (b) in Section~\ref{sec:3-4} with the variable $Z$ deleted.

\subsection{WCC vs Modern Language Models}

Our WCC framework provides a unifying perspective that encompasses both classical embedding models and modern large language models. While contemporary large language models \cite{devlin2019bert,leenv} employ sophisticated architectures, their training objectives can be effectively interpreted through the WCC lens. Although the success of these models largely stems from complex neural architectures (e.g., Transformer blocks) rather than negative sampling alone, such advanced architectures naturally fit within the comprehensive WCC framework.

More precisely, the functions $f$, $g$, and $s$ introduced in Section~\ref{sec:32} represent simplified instantiations that can be substantially enhanced. For example, the scoring function $s$ can be implemented as a learnable neural network rather than a simple inner product. When $f$ incorporates modern architectures like Transformers, which are the fundamental building blocks of contemporary large language models, WCC effectively describes the pre-training process and the resulting representations. Since our theoretical analysis remains architecture-agnostic, the core theoretical results maintain their validity across different model implementations.

Furthermore, modern language modeling objectives can be viewed as specialized instances of the Continuous Bag of Words (CBOW) formulation with expanded context definitions and modified window mechanisms, which similarly accommodate negative sampling strategies. For instance, in autoregressive language modeling tasks, one could replace conventional Softmax layers with negative sampling schemes, while in contrastive learning objectives common in large model training, one could maximize similarity between appropriate sequence pairs while minimizing similarity with strategically sampled negatives. In such scenarios, our analytical framework offers valuable theoretical guidance. We defer large-scale pre-training experiments leveraging these insights to researchers with adequate computational resources.

\subsection{Performance gain of WCC over existing models}
The main objective of this paper is to present a unified principle for SGN-like models, not to develop an LLM-like ``super-model'' for word embedding. Word embedding and, more generally, representation learning, have witnessed great success in recent years. However, the fundamental principles underlying these models are poorly explored. For example, even the problem of ``representation learning'' is poorly defined: Except by relying on some downstream tasks to evaluate a learned representation, very little has been developed pertaining to metrics or principles for word-embedding alike representation learning; what makes a ``good'' representation remains elusive. In this respect, Corollary \ref{cor:justifyWCC} of this paper and PMI-MF are among the only results, to our knowledge.

Despite our theoretical focus, this paper does yield new models that outperform existing models. Compared to 3/4-ugSGN, our models not only perform significantly better, but they are also more ``negative-sample efficient'': our adaptive samplers draw one noise sample per sample, whereas 3/4-ugSGN draws 5 noise samples per sample. The fact that ACE also performs well should not negatively affect our results. As stated above, ACE is essentially an adaptive conditional member of WCC. Its good performance further endorses some claims of our paper, namely that WCC is a useful generalization and that adaptive conditional schemes can be advantageous. However, the new models discovered still significantly outperform ACE in some tasks.

\section{Conclusion}\label{conclusion}
In this paper, we introduce the WCC framework for word embedding that generalizes SGN to a much wider family. We provide a theoretical analysis 
that justifies the framework. The well-known matrix-factorization result of \cite{LevyG14} can be recovered from this analysis. We experimentally study the impact of noise distribution in the framework. Our experiments confirm the hypothesis that the best noise distribution is in fact the data distribution.  Along our way, novel word embedding models are developed and shown to outperform the existing models in the WCC family.
Looking forward, our work opens several promising research directions. Architecturally, applying WCC principles to Transformer-based architectures could yield more efficient large-scale models. Theoretically, further formalizing the relationship between noise distribution properties and embedding quality metrics would deepen our understanding. Practically, extending the WCC framework to multimodal representation learning could broaden its applicability. These directions collectively advance toward more principled and effective representation learning methodologies.

\bibliographystyle{IEEEtran}

\bibliography{cas-refs}


 





\end{document}

%% file: plot_ws.tex
\begin{tikzpicture}
\begin{axis}
[legend style={at={(0.95,0.4)},nodes={scale=0.9, transform shape}},
ymin=0.29,
xlabel=iteration ($\times 10000$),
  ylabel=Spearman's $\rho$]
\addplot[color=blue, mark=square] table [y=SG5, x=iter]{plot_wss.txt};
\addlegendentry{3/4-unigram SGN}
\addplot[color=red, dashed, mark=*] table [y=SGU, x=iter]{plot_wss.txt};
\addlegendentry{uniform SGN}
\addplot[color=black, mark=otimes] table [y=cASG1, x=iter]{plot_wss.txt};
\addlegendentry{caSGN2($n_{critic}$=1)}
\addplot[color=green, mark=triangle] table [y=cASG5, x=iter]{plot_wss.txt};
\addlegendentry{caSGN2($n_{critic}$=5)}
\addplot[color=RYB3, mark=diamond] table [y=SGUG, x=iter]{plot_wss.txt};
\addlegendentry{unigram SGN}
\end{axis}
\end{tikzpicture}

%% file: plot_rw.tex
\begin{tikzpicture}
\begin{axis}
[legend style={at={(0.95,0.4)},nodes={scale=0.9, transform shape}},
xlabel=iteration ($\times 10000$),
  ylabel=Total Accuracy],
\addplot[color=blue, mark=square] table [y=SG5, x=iter]{plot_an.txt};
\addlegendentry{3/4-unigram SGN}
\addplot[color=red, dashed, mark=*] table [y=SGU, x=iter]{plot_an.txt};
\addlegendentry{uniform SGN}
\addplot[color=black, mark=otimes] table [y=cASG1, x=iter]{plot_an.txt};
\addlegendentry{caSGN2($n_{critic}$=1)}
\addplot[color=green, mark=triangle] table [y=cASG5, x=iter]{plot_an.txt};
\addlegendentry{caSGN2($n_{critic}$=5)}
\addplot[color=RYB3, mark=diamond] table [y=SGUG, x=iter]{plot_an.txt};
\addlegendentry{unigram SGN}
\end{axis}
\end{tikzpicture}